\documentclass[english,nohyperref]{article}
\usepackage[T1]{fontenc}
\usepackage[utf8]{inputenc}
\usepackage{geometry}
\usepackage{babel}
\usepackage{float}
\usepackage{calc}
\usepackage{mathtools}
\usepackage{amsmath}
\usepackage{amsthm}
\usepackage{amssymb}
\usepackage{graphicx}
\usepackage[authoryear,round]{natbib}
\usepackage[unicode=true,pdfusetitle,
 bookmarks=true,bookmarksnumbered=false,bookmarksopen=false,
 breaklinks=false,pdfborder={0 0 1},backref=false,colorlinks=false]
 {hyperref}

\makeatletter

\floatstyle{ruled}
\newfloat{algorithm}{tbp}{loa}
\providecommand{\algorithmname}{Algorithm}
\floatname{algorithm}{\protect\algorithmname}

\theoremstyle{plain}
\newtheorem{thm}{\protect\theoremname}
\theoremstyle{plain}
\newtheorem{lem}[thm]{\protect\lemmaname}

\usepackage{algorithmic}

\makeatother

\providecommand{\lemmaname}{Lemma}
\providecommand{\theoremname}{Theorem}

\begin{document}
\global\long\def\EST{\mathrm{PRD}}%
\global\long\def\ALG{\mathrm{ALG}}%
\global\long\def\OPT{\mathrm{OPT}}%
\global\long\def\eb{e_{B}}%
\global\long\def\eba{e_{B_{a}}}%
\global\long\def\norm{\frac{\eba^{\alpha/B_{a}}-1}{\eba^{\alpha}-1}}%
\global\long\def\apred{a_{(\EST)}}%
\global\long\def\aalgo{a_{(\mathrm{EXP})}}%
\global\long\def\wi#1{w_{at_{#1}}}%
\global\long\def\betai#1{\beta_{a}^{(t_{#1})}}%
\global\long\def\X{\mathbf{X}}%
\global\long\def\P{\mathbf{P}}%
\global\long\def\S{\mathbf{S}}%
\global\long\def\t{t}%
\global\long\def\i{i}%
\global\long\def\j{j}%

\title{Online Ad Allocation with Predictions\thanks{This work was supported in part by NSF CAREER grant CCF-1750333, NSF
grant III-1908510, and a Sloan Foundation fellowship.}}
\author{Fabian Spaeh\thanks{Department of Computer Science, Boston University. ${\tt fspaeh@bu.edu}$}
\and Alina Ene\thanks{Department of Computer Science, Boston University. ${\tt aene@bu.edu}$}}
\maketitle
\begin{abstract}
Display Ads and the generalized assignment problem are two well-studied
online packing problems with important applications in ad allocation
and other areas. In both problems, ad impressions arrive online and
have to be allocated immediately to budget-constrained advertisers.
Worst-case algorithms that achieve the ideal competitive ratio are
known, but might act overly conservative given the predictable and
usually tame nature of real-world input. Given this discrepancy, we
develop an algorithm for both problems that incorporate machine-learned
predictions and can thus improve the performance beyond the worst-case.
Our algorithm is based on the work of \citet{feldman09} and similar
in nature to \citet{mahdian07} who were the first to develop a learning-augmented
algorithm for the related, but more structured Ad Words problem. We
use a novel analysis to show that our algorithm is able to capitalize
on a good prediction, while being robust against poor predictions.
We experimentally evaluate our algorithm on synthetic and real-world
data on a wide range of predictions. Our algorithm is consistently
outperforming the worst-case algorithm without predictions.
\end{abstract}

\section{Introduction}

Advertising on the internet is a multi-billion dollar industry with
ever growing revenue,  especially as online retail gains more and
more popularity. Typically, a user arrives on a website which fills
an empty advertising spot (called an impression) by allocating it
instantly to one of many advertisers. Advertisers value users differently
based on search queries or demographic data and reveal their valuations
in auctions or through contracts with the website.  Formulations
of increasing complexity have been studied to capture this problem,
creating a hierarchy of difficulty~\citep{mehta13}. The most basic
is online bipartite matching, where each vertex on one side  arrives
online with all its adjacent edges and has to be matched immediately
to one of the vertices on the other side, which were supplied offline.
The problem and all generalizations admit a hard lower bound of $1-\frac{1}{e}$
due to the uncertainty about future vertices. Motivated by online
ad exchanges, where advertisers place bids on impressions, \citet{mehta07}
introduced the Ad Words problem, which is a generalization of online
bipartite matching where we charge each advertiser for the amount
they bid. Going beyond the Ad Words setting, \citet{feldman09} considered
the more expressive problems Display Ads and the generalized assignment
problem (GAP), and proposed algorithms for these settings with worst-case
guarantees. 

 Classic algorithms that defend against the worst case of $1-\frac{1}{e}$
are often overly conservative given that the real world does not behave
like a contrived worst-case instance.  Recently, researchers have
thus been trying to  leverage a prediction about some problem parameter
to go beyond the worst-case \citep{mitzenmacher22}. In the context
of ad allocation, a prediction can be the keyword distribution of
users on a certain day, or simply the advertiser allocation itself.
Such a prediction is readily obtainable in practice, for example through
learning on historic data. Two opposing properties are important:
The algorithm has to be consistent, meaning that its performance should
improve with the prediction quality. Simultaneously, we want the algorithm
to be robust against a poor prediction, i.e. not to decay completely
but retain some form of worst-case guarantee. This is particularly
important in the advertising business as much revenue is extracted
from fat tails containing special events that are difficult or impossible
to predict, but extremely valuable for the advertising business (e.g.
advertising fan merchandise after a team's victory). To this end,
\citet{mahdian07} developed a learning-augmented algorithm for the
Ad Words problem and \citet{medina17} show how to use bid predictions
to set reserve prices for ad auctions.  Inspired by their work, we
develop a learning-augmented algorithm for Display Ads and GAP. 

\paragraph{Our Contributions}

We design the first algorithms that incorporate machine-learned predictions
for the well-studied problems Display Ads and GAP. The two problems
are general online packing problems, that capture a wide range of
applications. There has been work on covering problems with predictions
by \citet{bamas20} who posed the existence of learning-augmented
algorithms for packing problems as an open question, which we partially
answer in this work.  Our algorithm follows a primal-dual approach,
which yields a combinatorial algorithm that is very efficient and
easy to implement. It is able to leverage predictions which can be
learned from historical data. Using a novel analysis, we show that
the algorithm is robust against bad predictions and able to improve
its performance with good predictions. In particular, we are able
to bypass the strong lower bound on the worst-case competitive ratio
for these problems. We experimentally verify the practical applicability
of our algorithm under various kinds of predictions on synthetic and
real-world data sets. Here, we observe that our algorithm is able
to outperform the baseline worst-case algorithm due to \citet{feldman09}
that does not use predictions, by leveraging predictions that are
obtained from historic data, as well as predictions that are corrupted
versions of the optimum allocation.

\subsection{Preliminaries}

\paragraph{Problem Definition}

In this work, we study the Display Ads problem and its generalization,
the generalized assignment problem (GAP) \citep{feldman09}. In Display
Ads, there are advertisers $a\in\left\{ 1,\dots,k\right\} $ that
are known ahead of time, each of which is willing to pay for at most
$B_{a}$ ad impressions. A sequence of ad impressions arrive online,
one at a time, possibly in adversarial order. When impression $t$
arrives, the values $w_{at}\ge0$ for each advertiser $a$ become
known. These values might be a prediction of click-through probability
or any valuation from the advertiser, but we treat them as abstractly
given to the algorithm. We have to allocate $t$ immediately to an
advertiser, or decide not to allocate it at all. The goal is to maximize
the total value $\sum_{a,t}x_{at}w_{at}$ subject to $\sum_{t}x_{at}\le B_{a}$
for all $a$, where $x_{at}=1$ if $t$ is allocated to $a$ and $x_{at}=0$,
otherwise. GAP is a generalization of Display Ads where the size that
each impression takes up in the budget constraint of an advertisers
is non-uniform. That is, each impression $t$ has a size $u_{at}\ge0$
for each advertisers $a$, and advertiser $a$ is only willing to
pay for a set of impressions whose total size is at most $B_{a}$.
More precisely, we require that $\sum_{a,t}x_{at}u_{at}\le B_{a}$.

\paragraph{Free Disposal Model}

In general, it is not possible to achieve any competitive ratio for
the online problems described above. Motivated by online advertising,
\citet{feldman09} introduced the free disposal model which makes
the problem tractable: when a new impression arrives, the algorithm
allocates it to an advertiser $a$; if $a$ is out of budget, we can
decide to dispose of an impression previously allocated to $a$. The
motivation for this model is that advertisers are happy to receive
more ads, as long as they are only charged for the $B_{a}$ most valuable
impressions. We refer the reader to the paper of \citet{feldman09}
for additional motivation of this model. In this work, we consider
both Display Ads and GAP in the free disposal model. 

\paragraph{Related Problems}

Display Ads and GAP are significant generalizations of well-studied
problems such as online bipartite matching and Ad Words. In online
bipartite matching, all values are $1$. In Ad Words, values and sizes
are identical. The latter setting allows for more specialized algorithms
that exploit the special properties of this problem, as we discuss
in more detail later. 

\paragraph{Algorithms with Predictions}

The algorithms we study follow under the umbrella of learning-augmented
algorithms that leverage machine learning predictions to obtain improved
performance. These were studied in an extensive line of work, see
e.g. the survey of \citet{mitzenmacher22}. Following this established
research, we use two important measures for the performance of the
algorithm: The \emph{robustness} $\ALG/\OPT$ indicates how well the
algorithm's objective value $\ALG$ performs against the optimum solution
$\OPT$; the \emph{consistency} $\ALG/\EST$ measures how close the
algorithm gets to the prediction's objective value $\EST$. Most learning-augmented
algorithms, including the one presented in this work, allow to control
the trade-off between robustness and consistency with a parameter
$\alpha$.

\subsection{Related Work }

\paragraph{Online Ad-Allocation with Predictions}

To the best to our knowledge, we are the first to study Display Ads
and GAP with predictions. Related problems were considered in the
work by \citet{mahdian07} and \citet{medina17} for Ad Words, and
by \citet{lattanzi20} for online capacitated bipartite matching.
 \citet{medina17} use bid predictions to set reserve prices for
ad auctions. \citet{lavastida21} incorporate predictions of the dual
variables into the proportional-weights algorithm \citep{agrawal18,karp90,lattanzi20}
for online capacitated bipartite matching. \citet{chen2021} analyze
an algorithm that uses degree predictions by matching arriving vertices
to vertices with minimum predicted degree. As noted above, Ad Words
and bipartite matching have additional structure, which is exploited
in these prior works. In particular, the algorithms proposed in these
works are not applicable to the more general problems Display Ads
and GAP. Our algorithm builds on the approach of \citet{mahdian07}
for the Ad Words problem, but substantial new ideas are needed in
the algorithm design and analysis, as discussed in more detail in
Section \ref{sec:algo}.

There has been further extensive work in the design of worst-case
algorithms and under random input models without predictions, which
we now summarize.

\paragraph{Worst-Case Algorithms}

The design of worst-case algorithms has been the focus of a long line
of work which can, for instance, be found in the survey of \citet{mehta13}.
A large focus has been on Ad Words. Several combinatorial algorithms
have been proposed, based on the work of \citet{karp90}. The combinatorial
approach is tailored to the structure of to these special cases. The
primal-dual approach is a more general approach that can handle more
complex problems such as Display Ads and GAP \citep{buchbinder07,feldman09}.
In this work, we build on the primal-dual algorithm of \citet{feldman09}
and show how to incorporate predictions into their framework. The
worst-case guarantee for online bipartite-matching, and therefore
for all generalizations, is $1-\frac{1}{e}$ \citep{karp90}. 

\paragraph{Stochastic Algorithms}

The lower bound of $1-\frac{1}{e}$ can be circumvented under distributional
assumptions. This has been extensively studied for online bipartite
matching \citep{karande11,feldman09-2,jin22}. Further work has been
done for the Ad Words problem \citep{devanur09,devanur12} with generalizations
due to \citet{feldman10} for a more general stochastic packing problem.

\section{\label{sec:our-algorithm} Our Algorithm}

\allowdisplaybreaks

\label{sec:algo}

In order to illustrate the algorithmic ideas and analysis, we consider
the simpler setting of Display Ads in this section. Our algorithm
for GAP is a generalization of this algorithm and we include it in
Appendix \ref{sec:gap}. For simplicity, we assume that our prediction
is a solution to the problem, which is as in prior work \citep{mahdian07}.
However, due to our general analysis framework, we also consider our
algorithm a starting point towards incorporating weaker predictors,
such as partial solutions or predictions of the supply, which we leave
for future work.

\paragraph{Prediction}

We assume that we are given access to a prediction, which is a fixed
solution to the problem, given as an allocation of impressions to
advertisers. With each impression $t$, we also receive the advertiser
$\EST(t)$ to which the prediction allocates $t$. In particular,
this means that the prediction does not have to be created up front,
but can be adjusted on the fly based on the observed impressions.

Given a solution to the problem, we could also consider the following
random-mixture algorithm: For some parameter $q\in[0,1]$, run the
worst-case algorithm; with probability $1-q$, follow the prediction
exactly. This algorithm achieves a robustness of $q\cdot(1-\frac{1}{e})$
and consistency of $q\cdot(1-\frac{1}{e})+1-q$. However, this is
only in expectation and against a weak adversary that is oblivious
to the algorithm's random choices. In contrast, Algorithm~\ref{alg:exp-avg}
is designed to obtain its guarantees deterministically against the
strongest possible adversary that can adapt to the algorithm's choices,
which is identical to the setup in \citet{mahdian07}. We observe
that Algorithm \ref{alg:exp-avg} clearly outperforms this random-mixture
algorithm in our experiments (cf. Section \ref{sec:experiments})
which shows that our stronger setting is indeed valuable in practice.
Furthermore, the random-mixture algorithm cannot be adapted to different
predictors that are not solutions, such as the ones mentioned above.

\begin{figure}
\begin{centering}
\noindent\fbox{\begin{minipage}[t]{1\columnwidth - 2\fboxsep - 2\fboxrule}%
\begin{center}
\vspace{-8pt}
\begin{minipage}[t]{0.3\columnwidth}%
\begin{align*}
\text{} & \text{Display Ads Primal}\\
\max & \sum_{a,\t}w_{a\t}x_{a\t}\\
\forall a\colon & \sum_{\t}x_{a\t}\leq B_{a}\\
\forall\t\colon & \sum_{a}x_{a\t}\leq1\\
\forall a,\t\colon & x_{a\t}\ge0
\end{align*}
\end{minipage}\hspace{2cm}%
\begin{minipage}[t]{0.3\columnwidth}%
\begin{align*}
 & \text{Display Ads Dual}\\
\min & \sum_{a}B_{a}\beta_{a}+\sum_{\t}z_{\t}\\
\forall a,\t\colon & z_{\t}\ge w_{a\t}-\beta_{a}\\
\forall a\colon & \beta_{a}\ge0\\
\forall\t\colon & z_{\t}\ge0
\end{align*}
\end{minipage}
\par\end{center}%
\end{minipage}}
\par\end{centering}
\caption{\label{fig:lp} Primal and dual of the Display Ads LP}
\end{figure}

\paragraph{Algorithm Overview}

\begin{algorithm}[t]
\noindent\begin{minipage}[t]{1\columnwidth}%
\begin{algorithmic}[1]
\STATE {\bf Input:} Robustness-consistency trade-off parameter $\alpha \in [1,\infty)$, advertiser budgets $B_a \in \mathbb N$
\STATE Define the constants $B \coloneqq \min_a B_a$,
	$e_{B} \coloneqq \left(1+\frac{1}{B}\right)^{B}$, and
	$\alpha_{B} \coloneqq B\left(\eb^{\alpha/B}-1\right)$
\STATE For each advertiser $a$, initialize $\beta_a \gets 0$ and allocate $B_a$ zero-value  impressions
\FORALL{arriving impressions $\t$}
\STATE $\apred \gets \EST(\t)$
\STATE $\aalgo \gets \arg\max_{a}\{w_{a\t}-\beta_{a}\}$ 
\IF{$\alpha_{B}(w_{\apred \t}-\beta_{\apred})\!\ge\!w_{\aalgo \t}-\beta_{\aalgo}$} \label{alg:selection-rule}
\STATE $a \gets \apred$
\ELSE
\STATE $a \gets \aalgo$
\ENDIF
\STATE Allocate $\t$ to $a$ and remove the least valuable impression currently assigned to $a$
\STATE Let $w_1 \le w_2 \le \cdots \le w_{B_a}$ be the values of impressions currently assigned to $a$ in non-decreasing order
\STATE Update $\displaystyle \beta_{a}\gets\frac{\eba^{\alpha/B_{a}}-1}{\eba^{\alpha}-1}\sum_{i=1}^{B_{a}}w_{i}\eba^{\alpha\left(B_{a}-i\right)/B_{a}}$  \label{alg:update}
\ENDFOR
\end{algorithmic}%
\end{minipage}

\caption{\label{alg:exp-avg} Exponential Averaging with Predictions}
\end{algorithm}
Our Algorithm, shown in Algorithm \ref{alg:exp-avg}, incorporates
predictions in the primal-dual algorithm of \citet{feldman09}. The
algorithm is based on the primal and dual LP formulations in Figure
\ref{fig:lp}. The algorithm constructs both a primal integral solution
which is an allocation of impressions to advertisers as well as a
dual solution, given explicitly by the dual variables $\beta_{a}$.
Analogously to other algorithms with predictions, the algorithm takes
as parameter the value $\alpha$; a larger value of $\alpha$ means
that we trust the prediction more. Similarly to the worst-case algorithm
of \citet{feldman09}, the dual variables $\beta_{a}$ play an important
role in the allocation of impressions. When an impression $t$ arrives,
we evaluate the discounted gain $w_{at}-\beta_{a}$ for each advertiser.
The worst-case algorithm allocates the impression to the advertiser
$\aalgo$ maximizing the discounted gain and only allocates if the
discounted gain is positive, i.e. the value exceeds the threshold
$\beta_{a}$. Our algorithm modifies this base algorithm to incorporate
predictions as shown in Line \ref{alg:selection-rule} and it follows
the prediction $\apred$ if its discounted gain is a sufficiently
high fraction of the discounted gain of $\aalgo$. We refer the reader
to the discussion below for more intuition on the choice of update.
After selecting the advertiser to which to allocate the impression
$\t$, we remove the least valuable impression currently assigned
to $a$ to make space for $\t$, and then allocate $\t$ to $a$.
Another crucial part of the algorithm is the update rule for $\beta_{a}$
in Line \ref{alg:update}, which is updated in a novel way based on
the parameter $\alpha$. More precisely, we set $\beta_{a}$ as an
exponential averaging of the values of impressions currently allocated
to $a$. Compared to the worst-case algorithm, we assign higher weight
to impressions with less value which is essential for leveraging predictions. 

To simplify the algorithm description and analysis, we initially allocate
$B_{a}$ impressions of zero value to each advertiser $a$. Furthermore,
we assume that there exists a ``dummy'' advertiser with large budget
that only receives zero value impressions. Instead of not allocating
an impression explicitly (either in the algorithm or the prediction),
we allocate to the dummy advertiser, instead.

\paragraph{Intuition for our Algorithm}

As noted above, we make two crucial modifications to the worst-case
algorithm to incorporate predictions: The advertiser selection (Line
\ref{alg:selection-rule}) and the update of $\beta_{a}$ (Line \ref{alg:update}).
We now provide intuition for both choices.

First, let us illustrate the difficulties in incorporating predictions
in the advertiser selection. Based on the worst-case algorithm, a
natural way to incorporate predictions is to allocate to the prediction
if the distorted gain $w_{\apred\t}-\frac{1}{\alpha}\cdot\beta_{\apred}$
exceeds the maximum discounted gain. However, this approach does not
work as shown by the following example: Consider a scenario where
the prediction suggests a constant advertiser $\apred$. Impressions
are split into two phases: Phase 1 contains $B_{\apred}$ impressions
$\t$ where only $\apred$ can derive a value of $w_{\apred\t}=1$
and $w_{a\t}=0$ for $a\not=\apred$. The algorithm allocates all
these impressions to $\apred$ and at the end of phase 1 has $\beta_{\apred}=1$.
In phase 2, a large amount of impressions with $w_{\apred\t}=1$ and
$w_{a\t}=\frac{\alpha-1}{2\alpha}$ for $a\not=\apred$ arrive. Since
the distorted gain $w_{\apred\t}-\frac{1}{\alpha}\beta_{\apred}=\frac{\alpha-1}{\alpha}$
exceeds the discounted gain $w_{a\t}-\beta_{a}=\frac{\alpha-1}{2\alpha}$
for $a\not=\apred$, the algorithm allocates to $\apred$ which yields
$0$ gain. However, we forfeit an unbounded amount of potential value
derived from allocating to advertisers $a\not=\apred$. An important
takeaway of this example is the crucial observation that the algorithm
should never allocate to the predicted advertiser if its discounted
gain is $0$. The selection rule in our algorithm is designed to meet
this important consideration.

Second, we need to change the update rule for $\beta_{a}$. We update
$\beta_{a}$ using a carefully selected exponential average of the
values of impressions currently assigned to $a$, that incorporates
our confidence in the prediction parameterized by $\alpha$. In contrast
to the worst-case algorithm, we weigh less valuable impressions more.
This lowers the threshold for the addition of new impressions, which
allows us to exploit more potential gain from the predicted advertiser. 

\begin{figure}
\begin{centering}
\includegraphics[width=0.85\linewidth]{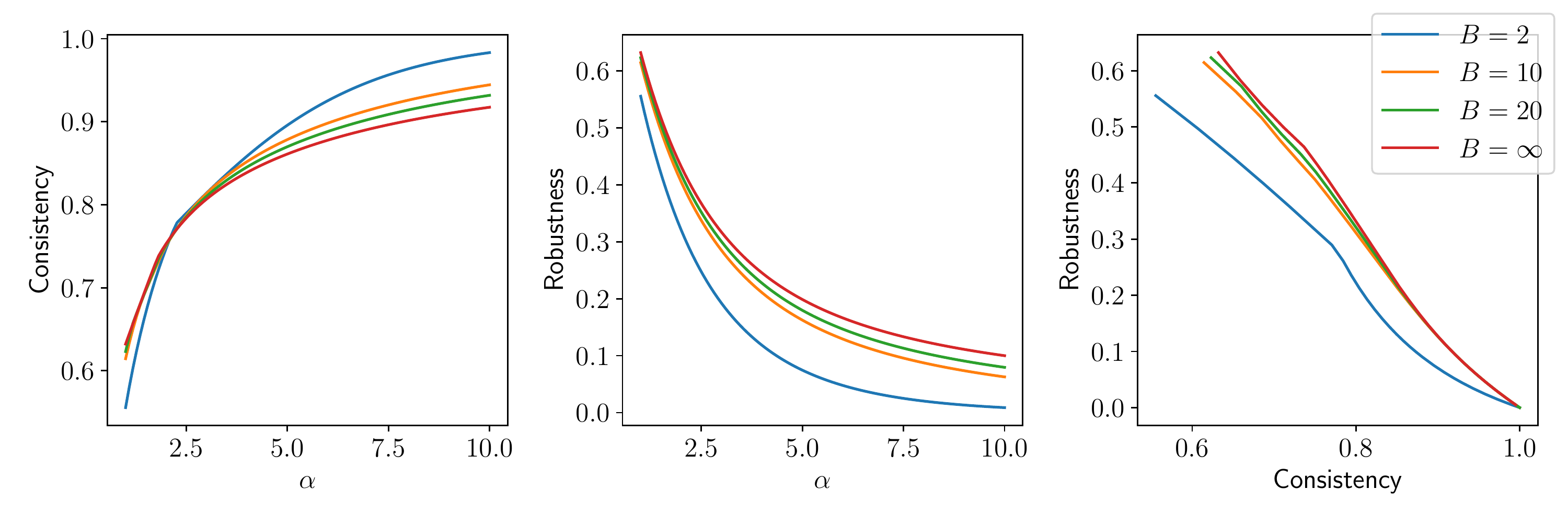}
\par\end{centering}
\vspace{-10pt}
\caption{\label{fig:cr-trade} We illustrate the consistency-robustness trade-off
of Algorithm \ref{alg:exp-avg} for various values of $\alpha$ and
budgets $B$.}
\end{figure}

\begin{thm}
\label{thm:main} Let $B\coloneqq\min_{a}B_{a}$ and $e_{B}\coloneqq\left(1+\frac{1}{B}\right)^{B}$.
Let $\OPT$ and $\EST$ be the values of the optimal and predicted
solutions, respectively. For any $\alpha\ge1$, Algorithm \ref{alg:exp-avg}
obtains a value of at least
\[
\ALG\ge\max\left\{ R(\alpha)\cdot\OPT,C(\alpha)\cdot\EST\right\} 
\]
where the robustness is
\[
R(\alpha)\coloneqq\frac{\eb^{\alpha}-1}{B\eb^{\alpha}\left(\eb^{\alpha/B}-1\right)}\to\frac{e^{\alpha}-1}{\alpha e^{\alpha}}\quad(B\to\infty)
\]
and the consistency is
\begin{align*}
C(\alpha)\coloneqq & \left(1+\frac{1}{\eb^{\alpha}-1}\max\left\{ \frac{1}{\alpha_{B}}\left(\eb^{\alpha}-\frac{\eb^{\alpha}-1}{\alpha_{B}}\right),\ln\left(\eb^{\alpha}\right)\right\} \right)^{-1}\\
 & \quad\to\left(1+\frac{1}{e^{\alpha}-1}\max\left\{ \frac{1}{\alpha}\left(e^{\alpha}-\frac{e^{\alpha}-1}{\alpha}\right),\alpha\right\} \right)^{-1}\quad(B\to\infty).
\end{align*}
\end{thm}

Figure \ref{fig:cr-trade} shows the above trade-off between consistency
and robustness. We can observe that the guarantee rapidly improves
as the minimum advertiser budget $B$ increases, which is very beneficial
both in theory and in practice. The trade-off is comparable to the
one obtained by \citet{mahdian07} for the more structured Ad Words
problem.

\paragraph{Comparison to Prior Work}

The work most closely related to ours is the work of \citet{mahdian07}
which incorporates predictions into the algorithm of \citet{mehta07}.
These works are for the related but different Ad Words problem, and
do not apply to Display Ads and GAP. In the Ad Words problem, the
value of an impression is equal to the price that the advertiser pays
for it, i.e., the size of the impression in the budget constraint.
In contrast, in Display Ads and GAP, the values and the sizes are
independent of each other (e.g., in Display Ads an impression takes
up $1$ unit of space in the advertiser's budget but it can accrue
an arbitrary value). Thus there is no longer any relationship between
the total value/profit of the impressions and the amount of the advertiser's
budget that has been exhausted. The Ad Words algorithms of \citet{mehta07,mahdian07}
crucially rely on this relationship both in the algorithm and in the
analysis. Due to the special structure of the problem, these algorithms
do not dispose of impressions and consider only the fraction of the
advertiser's budget that has been filled up in order to decide the
allocation and to incorporate the prediction. Moreover, the algorithm
and the analysis do not need to account for the loss incurred by disposing
impressions. These crucial differences require a new algorithmic approach
and analysis, which was given by \citet{feldman09} using a primal-dual
approach. Since we build on their framework as opposed to the work
of \citet{mehta07}, we also need a new approach for incorporating
predictions, as described above in the intuition for our algorithm.
Further, the primal-dual framework only helps in establishing the
robustness but not the consistency of our algorithm, and we develop
a novel combinatorial analysis for proving the consistency.

\section{Analysis}

In the following, we outline the analysis of Algorithm \ref{alg:exp-avg}
to prove Theorem \ref{thm:main}. Specifically, we show separately
that $\ALG/\OPT\le R(\alpha)$ (robustness) and $\ALG/\EST\le C(\alpha)$
(consistency).

\paragraph{Notation}

We denote with superscript $(\t)$ the value of variables after allocating
impression $\t$. E.g. $a^{(\t)}$ is the algorithm's choice of advertiser
for impression $\t$ and $\beta_{a}^{(\t)}$ is the dual variable
after allocating $\t$. Let
\[
\X_{a}\coloneqq\big\{\t:a^{(\t)}=a\big\}\quad\textrm{and}\hspace{1em}\P_{a}\coloneqq\big\{\t:\apred^{(\t)}=a\big\}
\]
be the impressions that were assigned to $a$ and potentially disposed
of, and the impressions that the prediction recommended for assignment
to $a$, respectively. We set $I_{a}\coloneqq\left|\X_{a}\right|$
and $\ell_{a}\coloneqq\left|\P_{a}\cap\X_{a}\right|$ as the size
of the overlap. Let also $T$ be the last impression and $\S_{a}$
is the final allocation, i.e., the $B_{a}$ impressions allocated
to $a$ at the end of the algorithm. Finally, let $\ALG$ and $\EST$
be the total value of the solution created by the algorithm and the
prediction, respectively. 

\paragraph{Robustness}

Our proof for robustness closely follows the analysis in \citet{feldman09}
by using the primal-dual formulation of the problem, with some additional
care that is needed not to violate dual feasibility whenever we follow
the prediction. We defer the full proof to Appendix \ref{subsec:omitted-robustness}. 

\paragraph{Consistency}

We now show the consistency, i.e. that $\EST$ is bounded by a multiple
of $\ALG$. The complete analysis can be found in Appendix \ref{subsec:omitted-consistency},
while we only give a high-level overview here.

A common approach in the analysis of online primal-dual algorithms
is to employ a local analysis where, in each iteration, we relate
the change in the value of the primal solution to the change in the
dual solution \citep{buchbinder09}. However, it is not clear how
to employ such a strategy in our setting due to the complications
arising from our algorithm following a mixture of the worst-case and
predicted solutions.  We overcome this challenge using a novel global
analysis that relates the final primal value to the prediction's value. 

We now provide a high level overview of our global analysis. We start
by noting that the objective value of our algorithm and the prediction
is the sum of the impression values allocated to each advertiser,
i.e.
\[
\ALG=\sum_{a}\sum_{t\in\S_{a}}w_{at}\qquad\text{and}\qquad\EST=\sum_{a}\sum_{t\in\P_{a}}w_{at}
\]
However, note that $\EST$ contains values of impression that do not
appear in $\ALG$ since we ignore the prediction in some iterations,
or already disposed of the impression.  It is further unclear which
advertiser to ``charge'' for an impression that does not agree with
the prediction.

Consider an impression for which we followed the worst-case choice
$\aalgo$ that maximizes the discounted gain instead of the prediction.
Due to our selection rule, the reason for this departure is due to
the discounted gains satisfying the following key inequality: 
\begin{equation}
w_{\apred}\le\frac{1}{\alpha_{B}}\left(w_{\aalgo}-\beta_{\aalgo}\right)+\beta_{\apred}.\label{eq:22}
\end{equation}
By using this important relationship, we upper bound the value $\EST$
of the prediction using a linear combination of the values of impressions
allocated by the algorithm (but possibly disposed of) and the thresholds.
By grouping the impression values and dual variables by advertiser
in the resulting upper bound, we are able to correctly charge each
impression for which we deviated from the prediction to a suitable
advertiser, thus overcoming one of the challenges mentioned above.
To summarize, using (\ref{eq:22}) we obtain a bound $\EST\le\sum_{a}\EST_{a}$
where each $\EST_{a}$ is a linear combination of impression values
and dual variables for advertiser $a$, and we want to compare this
quantity to $\ALG_{a}\coloneqq\sum_{t\in S_{a}}w_{at}$.

Let us now consider a fixed advertiser $a$, and relate $\EST_{a}$
to $\ALG_{a}$. At this point, a key difficulty is that the amount
$\EST_{a}$ that we charged to advertiser $a$ involves the threshold
$\beta_{a}$. By definition, the threshold is a convex combination
of the values of the impressions in the algorithm's allocation. This
gives us that $\EST_{a}$ is a linear combination of only the weights,
but this cannot be readily compared to $\ALG_{a}$ due to the complicated
structure of the coefficients in the former. To bridge this gap, we
show a useful structural property (Lemma \ref{lem:wlog}) that gives
us the following upper bound on $\EST_{a}$: If we define $\t_{\i}$
as the $\i$-th impression allocated to $a$, we have
\begin{equation}
\EST_{a}\le\sum_{\i=I_{a}-B_{a}+1}^{I_{a}-\ell_{a}}\phi_{\i}\wi{\i}+\sum_{\i=I_{a}-\ell_{a}+1}^{I_{a}}\psi_{\i}\wi{\i}+\wi{I_{a}-B_{a}}\Omega_{a}\label{eq:23}
\end{equation}
for appropriate coefficients $\phi_{i}$, $\psi_{i}$, and $\Omega_{a}$.
The RHS of this inequality accounts for the value as follows: the
first sum is for the impressions that agree with the prediction; the
second sum is for the impressions that disagree with the prediction;
the final term accounts for the values of all impressions that were
disposed. As this is (almost) a linear combination over impression
values that all appear in $\ALG$ (except $w_{at_{I_{a}-B_{a}}}$),
we could bound the ratio $\EST_{a}/\ALG_{a}$ by the maximum coefficient
in (\ref{eq:23}). However, this does not lead to a constant competitive
ratio. We therefore need a delicate analysis (Lemma \ref{lem:avgfac})
to balance the coefficients as uniformly as possible among all values,
where we use properties of the coefficients $\phi_{i}$, $\psi_{i}$,
and $\Omega_{a}$ and the structural property we derived in Lemma
\ref{lem:wlog}.

\section{Experimental Evaluation}

\begin{figure}[t]
\begin{centering}
\vspace{-4pt}
\includegraphics[width=0.95\linewidth]{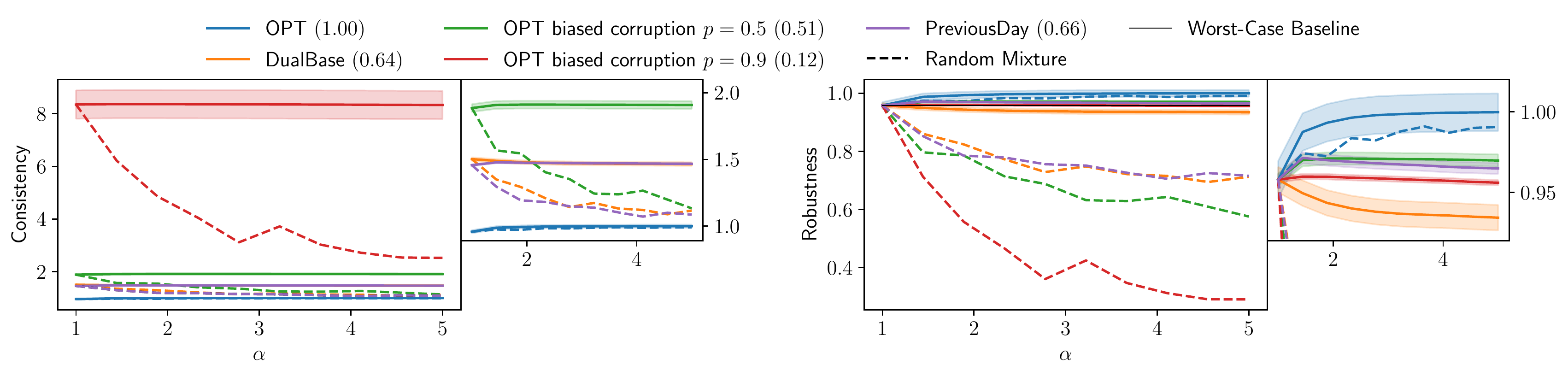}\vspace{-4pt}
\par\end{centering}
\begin{centering}
\includegraphics[width=0.95\linewidth]{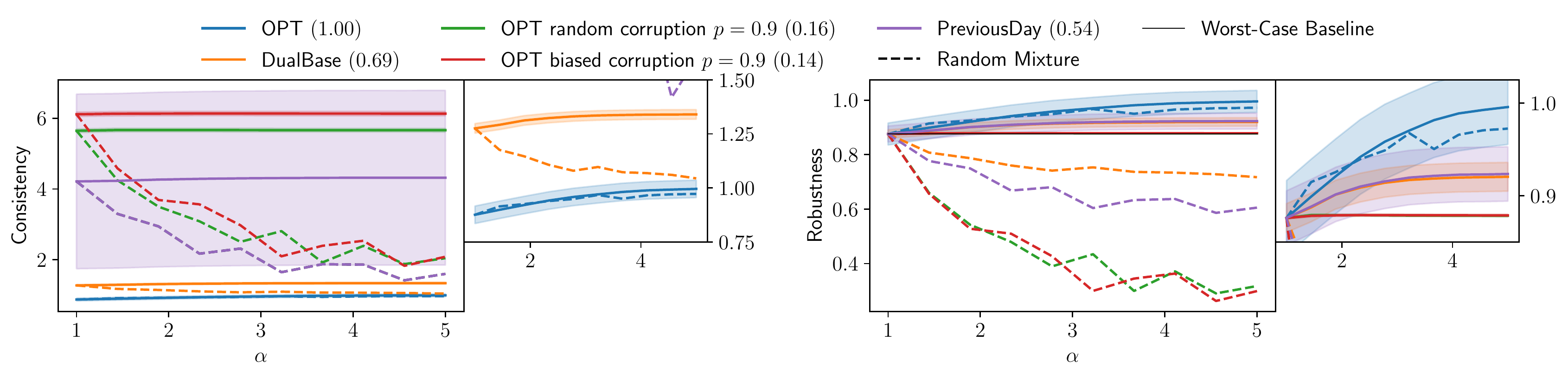}\vspace{-8pt}
\par\end{centering}
\caption{\label{fig:real-world-1} Experimental results on iPinYou (top) and
Yahoo datasets (bottom) using different predictions for varying $\alpha$.
The solid lines show our algorithm and the dashed lines the random-mixture
algorithm. We run the algorithms 5 times and report average for both
algorithms and the standard deviation only for our algorithm, to avoid
clutter. For the robustness, the black line shows the performance
of the worst-case algorithm without predictions due to \citet{feldman09}.
For each predictor, we also include in parentheses the average competitive
ratio over 5 runs (e.g. PreviousDay (0.66) indicates that the average
competitive ratio for the solution of the Previous Day prediction
was $0.66$). We run the random-mixture algorithm for each prediction
and $q\protect\coloneqq1/\alpha$.}
\end{figure}
\begin{figure}[t]
\begin{centering}
\vspace{-4pt}
\includegraphics[width=0.8\linewidth]{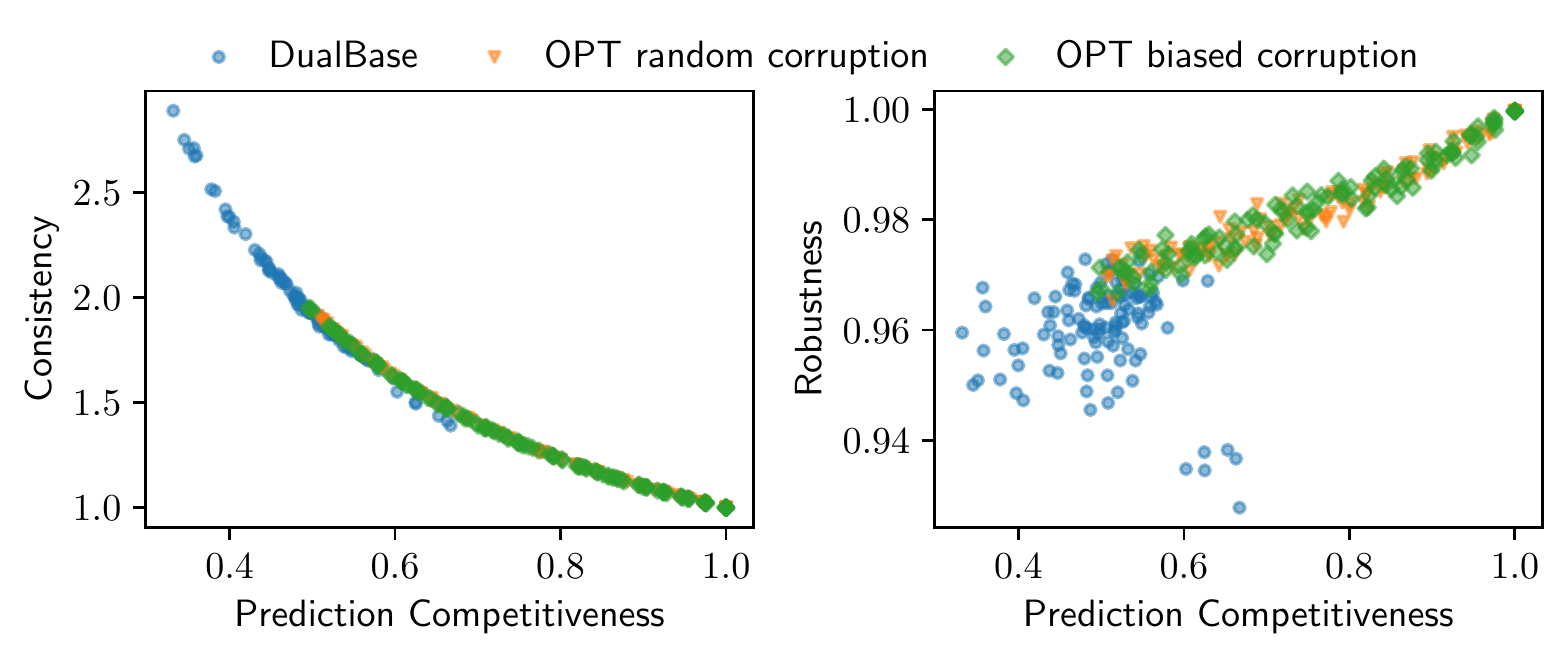}
\par\end{centering}
\begin{centering}
\includegraphics[width=0.8\linewidth]{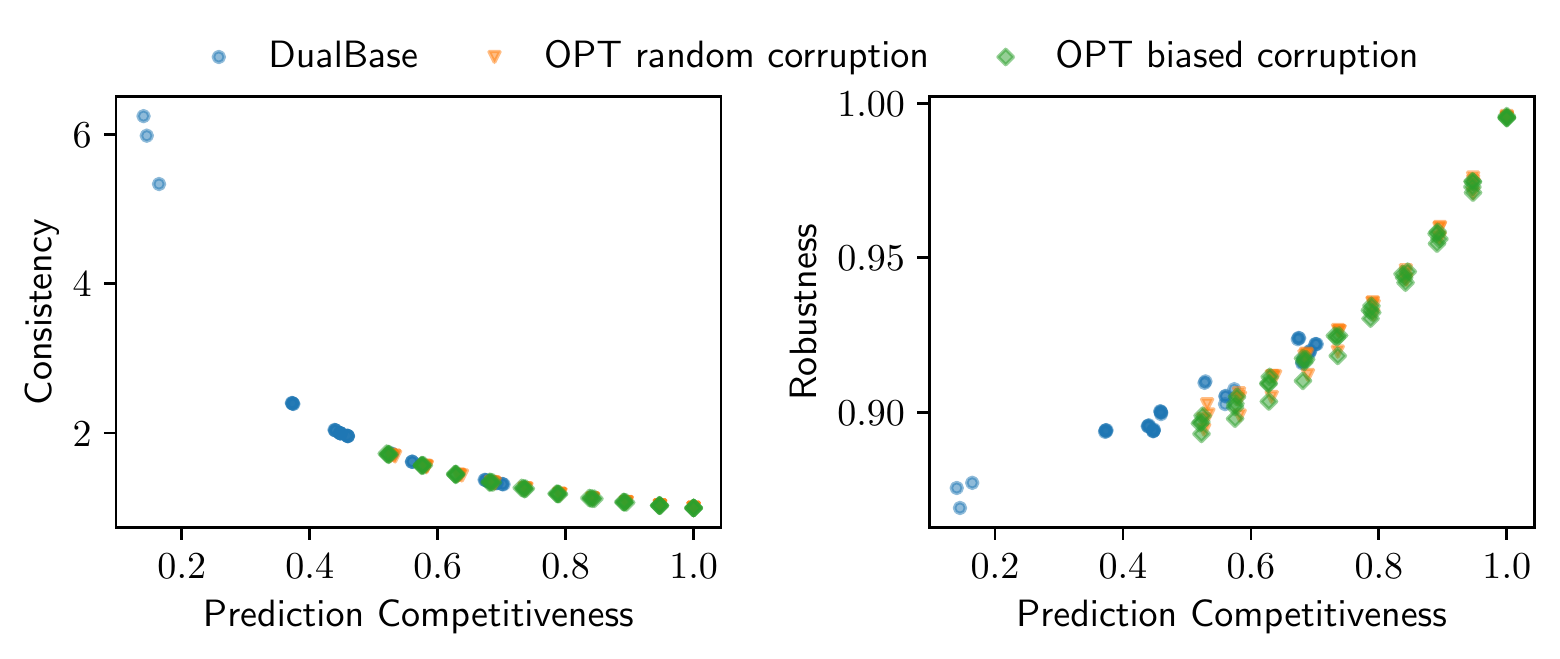}\vspace{-8pt}
\par\end{centering}
\caption{\label{fig:real-world-2} Performance of our algorithm on the iPinYou
(top) and Yahoo (bottom) datasets for $\alpha=5$ with predictions
of varying quality obtained as follows: We vary the sample fraction
$\epsilon\in\left[0,1\right]$ for the dual base algorithm and $p\in\left[0,1\right]$
for random and biased corruptions.}
\end{figure}
\begin{figure}[t]
\begin{centering}
\vspace{-4pt}
\includegraphics[width=0.85\linewidth]{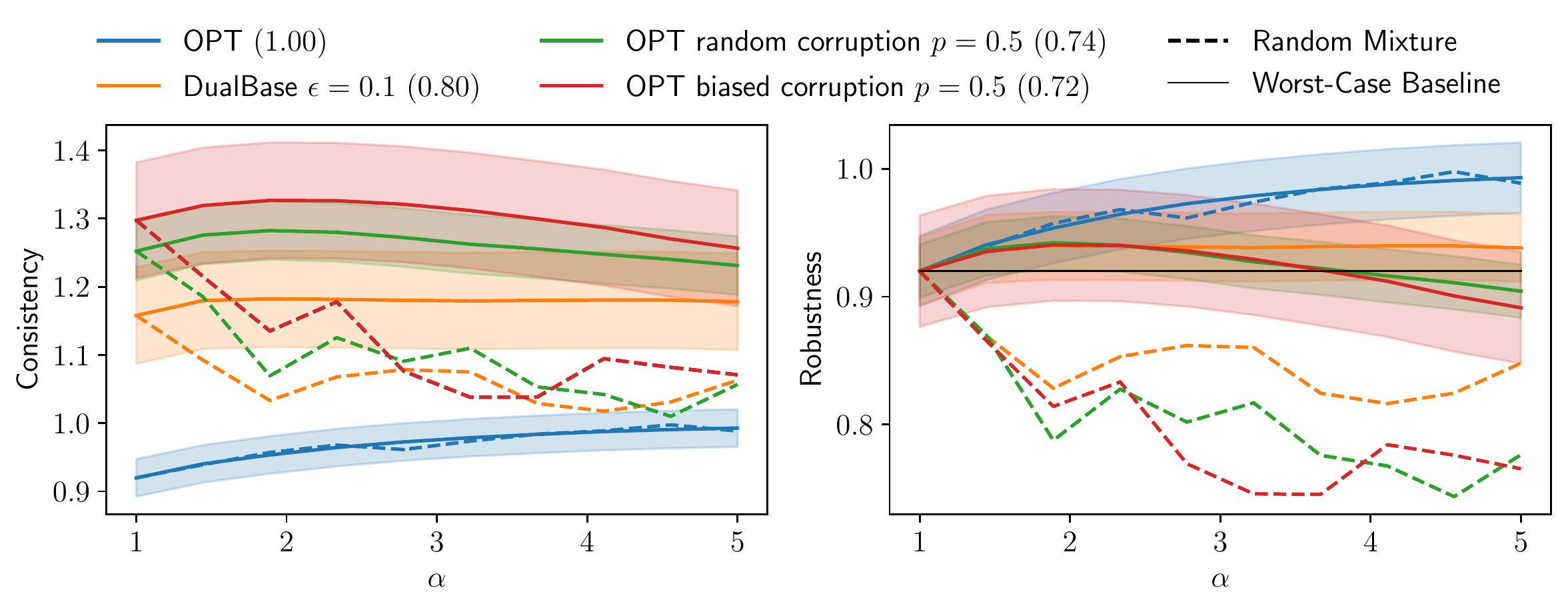}
\vspace{-8pt}
\par\end{centering}
\caption{\label{fig:synthetic-1-1-1} Experimental results for varying values
of $\alpha$ on synthetic data with $12$ advertisers and $2000$
impressions of $10$ types, where we report the same quantities as
in Figure \ref{fig:real-world-1}. We use different predictors with
$\sigma=1.5$. }
\end{figure}

\label{sec:experiments}
\global\long\def\numimps{T}%

We now evaluate the practical applicability of Algorithm~\ref{alg:exp-avg}.
We compare Algorithm~\ref{alg:exp-avg} to the worst-case algorithm
without predictions due to \citet{feldman09} and the random-mixture
algorithm described in Section \ref{sec:our-algorithm}. We use multiple
forms of predictions, which we describe below. We showcase results
on real-world and synthetic data, with further experimental results
in Appendix \ref{sec:apx-exp}.

\paragraph{Predictors}

We consider variations of the following predictors. Recall that each
predictor is a fixed allocation of impressions to advertisers that
is revealed online. 
\begin{enumerate}
\item \emph{Optimum Solution ($OPT$):} The optimum solution is obtained
by solving the problem optimally offline using an LP solver. To evaluate
our algorithm's robustness, we also consider a version of the optimum
solution where a random $p$-fraction of the allocations has been
corrupted. Under a \emph{random corruption}, we corrupt by reallocating
to randomly chosen advertisers. For a \emph{biased corruption}, we
sample a random permutation offline and corrupt by reallocating according
to this permutation, generating a more adversarial corruption.
\item \emph{Dual Base: } We generate a solution using the algorithm of
\citet{devanur09}\emph{. }Here, we sample the initial $\epsilon$-fraction
of all impressions and optimally solve a scaled version of the dual
LP to obtain the dual variables $\left\{ \beta_{a}\right\} _{a}$.
We get a primal allocation for all future impressions $t$ by allocating
to the advertiser $a$ that maximizes the discounted gain $w_{at}-\beta_{a}$,
but do not update $\beta_{a}$. 
\item \emph{Previous Day:} We look at all impressions from the previous
day and optimally solve the dual LP offline to obtain dual variables
$\left\{ \beta_{a}\right\} _{a}$. To get a prediction for today's
impressions, we use the same algorithm as above and allocate to the
advertiser maximizing the discounted gain. 
\end{enumerate}

\paragraph{Real-World Instances}

We generate two instances for Display Ads based on the real-word datasets
iPinYou \citep{zhang14} and Yahoo \citep{yahoo}.
\begin{enumerate}
\item \emph{iPinYou:} The iPinYou dataset contains real-time bidding data
from the iPinYou online advertisement exchange. This dataset contains
40372 impressions over 7 days and 301 advertisers. Each advertiser
places multiple bids for an impression. We use this bid data to construct
the dataset. Specifically, we set the maximum of those bids as the
advertiser's valuation. We assume a constant budget for each advertiser
of 10 impressions as it makes for an interesting instance.
\item \emph{Yahoo:} We replicate the experimental setup of \citet{lavastida21}
who generated an instance of online capacitated bipartite matching
based on a Yahoo dataset \citep{yahoo}. Capacitated online bipartite
matching is a special case of Display Ads where all impression values
are $1$. Based on this dataset, we create an instance of capacitated
online bipartite matching with around 2 million impressions and 16268
advertisers for 123 days. We defer the details to Section~\ref{subsec:yahoo-description}.
\end{enumerate}

\paragraph{Synthetic Instances}

We obtain random synthetic data for a fixed set of $\numimps$ impressions
and $k$ advertisers as follows. We first generate a set of impression
types, whereas each type is meant to model a group of homogenous users
(e.g. similar demographic or using similar keywords) and advertisers
value users from the same group identically. We sample an advertiser's
valuation for each impression type from an exponential distribution.
To represent a full day of impressions, we assume that display times
of impressions from a certain type are distributed according to a
Gaussian with some uniformly random mean in $\left[0,1\right]$ and
a fixed standard deviation $\sigma$. We then sample the same number
of impressions from each type along with display times, and order
them in increasing display time. Finally, we equip each advertiser
with some fixed budget that makes for a difficult instance.

\paragraph{Results}

Figure \ref{fig:real-world-1} and Figure \ref{fig:synthetic-1-1-1}
show results for real-world and synthetic instances, respectively.
For each predictor, we show the consistency (left) and robustness
(right) for varying $\alpha$. Figure \ref{fig:real-world-2} shows
results for $\alpha=5$ with predictions of different quality,  as
described in the figure caption. 

\paragraph{Discussion}

We make several observations. On real-world instances, there is only
a single prediction for which the performance of our algorithm drops
below the worst-case algorithm, even for heavily corrupted predictions.
E.g., on the iPinYou dataset, our algorithm is still able to leverage
a prediction with a corruption rate as high as $p=50\%$, and improve
upon the worst-case algorithm (see the green and black lines in the
top right plot of Figure \ref{fig:real-world-1}). Moreover, for higher
performing predictors, the improvement over the worst-case algorithm
is significantly higher in both datasets. See for example, the Previous
Day predictions on the iPinYou dataset (the purple and black lines
in the top right plot of Figure \ref{fig:real-world-1}) or Dual Base
predictions on the Yahoo dataset (the orange and black lines in the
bottom right plot of Figure \ref{fig:real-world-1}). Second, as we
can see in Figure \ref{fig:real-world-1}, the consistency of our
algorithm for predictors except the optimum is always above 1, and
is significantly high for artificially corrupted predictions. The
robustness of our algorithm remains high in almost all cases, even
for the most heavily corrupted predictions (cf. the right side of
Figure \ref{fig:real-world-2}).   On synthetic instances, we observe
that our algorithm is robust against both random and biased corruption,
as the robustness does not drop to the prediction's low competitiveness
of $\approx0.7$. Furthermore, our algorithm performs well in combination
with the dual base prediction for $\epsilon=0.1$ (the orange line
in Figure \ref{fig:synthetic-1-1-1}), even though the first $200$
impressions are not representative of all impressions. On all instances,
we clearly outperform the random-mixture algorithm which merely interpolates
between the objective values of the worst-case algorithm and the prediction.

\section*{Conclusion}

We introduce a novel learning-augmented algorithm for Display Ads
and GAP with free disposal. Our algorithm is based on the primal-dual
approach and can be efficiently implemented in practice. We show its
robustness using the primal-dual method similar to \citet{feldman09}
and use a novel combinatorial proof to show its consistency. Finally,
our experiments show the applicability of our algorithm, which is
able to improve beyond the worst-case performance using readily available
predictions. \textbf{Limitations:} Our algorithm requires a strong
prediction that is a solution to the problem. We leave weaker predictions,
such as partial solutions or predictions of the supply, for future
work.

\bibliographystyle{plainnat}
\bibliography{exp-avg-arxiv}

\newpage

\appendix

\section{Omitted Proofs}

\label{sec:omitted}

We will need the following helper Lemma in the proofs of consistency
and robustness.
\begin{lem}
\label{lem:param-props} Recall that $e_{B}\coloneqq\left(1+\frac{1}{B}\right)^{B}$
and $\alpha_{B}\coloneqq B\left(\eb^{\alpha/B}-1\right)$. We have
\begin{enumerate}
\item $\eb\le\eba$ and
\item $\alpha_{B}\ge\alpha_{B_{a}}.$
\end{enumerate}
\end{lem}

\begin{proof}
It is well known that $\eb$ converges to $e$ from below for $B\to\infty$.
Furthermore, we can show that $\alpha_{B}$ is decreasing in $B$
for all $\alpha\ge1$ by taking the derivative
\begin{align*}
\frac{\partial}{\partial B}\alpha_{B}=\left(1+\frac{1}{B}\right)^{\alpha}-1-\frac{\alpha}{B}\left(1+\frac{1}{B}\right)^{\alpha-1} & =\left(1+\frac{1}{B}\right)^{\alpha-1}\left(1+\frac{1}{B}\left(1-\alpha\right)\right)-1\\
 & \le\left(1+\frac{1}{B}\right)^{\alpha-1}\left(1+\frac{1}{B}\right)^{1-\alpha}-1=0
\end{align*}
where the bound follows from Bernoulli's inequality, which states
that $1+rx\le\left(1+x\right)^{r}$ for $x\ge-1$ and $r\in\mathbb{R}\setminus(0,1)$.
\end{proof}

\subsection{\label{subsec:omitted-robustness} Proof of Theorem \ref{thm:main}
(Robustness)}

We write $P$ and $D$ to denote the objective value of the primal
and dual solutions, i.e. $P=\sum_{a}\sum_{t\in\S_{a}}w_{at}$ and
$D=\sum_{a}B_{a}\beta_{a}+\sum_{\t}z_{\t}$ where $z_{t}$ is specified
in the following proof to ensure feasibility. We can show that after
the allocation of each impression $t$,
\[
\Delta P\ge\frac{\eb^{\alpha}-1}{B\eb^{\alpha}\left(\eb^{\alpha/B}-1\right)}\Delta D
\]
where $\Delta P$ and $\Delta D$ are the increase in the primal and
dual solution values, respectively. Since we create feasible primal
and dual solutions, this is sufficient to bound the robustness due
to weak duality. There is one main difference to \citet{feldman09}:
In their algorithm, setting the dual variable $z_{\t}$ to $w_{\aalgo t}-\beta_{a}$
ensures dual feasibility as $\aalgo$ is the advertiser with maximum
discounted gain. However, in order not to violate dual feasibility
when following the prediction, we need to increase the dual variables
$z_{\t}$ by a factor of $\alpha_{B}$. Note that for $\alpha=1$,
this recovers the competitiveness obtained by \citet{feldman09}.

\begin{proof}
Consider an iteration where we assign an impression $t$ to advertiser
$a$ and let $w_{1}\le w_{2}\le\cdots\le w_{B_{a}}$ be the values
of impressions currently allocated to $a$ in non-decreasing order.
Let $w_{0}$ be the least valuable of the impressions allocated to
$a$ at the end of iteration $t-1$, i.e. the impression that is removed
to make space for $t$. Assume that after allocating impression $t$
to $a$, it becomes the $k$-th least valuable impression allocated
to $a$ with value $w_{at}=w_{k}$. Thus, using that $w_{i}\ge w_{i-1}$,
we can bound
\begin{align*}
\beta_{a}^{(t-1)} & =\frac{\eba^{\alpha/B_{a}}-1}{\eba^{\alpha}-1}\left(\sum_{i=0}^{k-1}w_{i}\eba^{\alpha\left(B_{a}-i-1\right)/B_{a}}+\sum_{i=k+1}^{B_{a}}w_{i}\eba^{\alpha\left(B_{a}-i\right)/B_{a}}\right)\\
 & =\frac{\eba^{\alpha/B_{a}}-1}{\eba^{\alpha}-1}\left(\sum_{i=0}^{B_{a}-1}w_{i}\eba^{\alpha\left(B_{a}-i-1\right)/B_{a}}+\sum_{i=k+1}^{B_{a}}\left(w_{i}-w_{i-1}\right)\eba^{\alpha\left(B_{a}-i\right)/B_{a}}\right)\\
 & \ge\frac{\eba^{\alpha/B_{a}}-1}{\eba^{\alpha}-1}\left(\sum_{i=0}^{B_{a}-1}w_{i}\eba^{\alpha\left(B_{a}-i-1\right)/B_{a}}\right)\eqqcolon\hat{\beta}_{a}^{(t-1)}
\end{align*}
which is tight when impression $t$ becomes the most valuable impression
assigned to $a$. We can now write $\beta_{a}^{(t)}$ as a function
of the bound $\hat{\beta}_{a}^{(t-1)}$:
\begin{align*}
\beta_{a}^{(t)} & =\frac{\eba^{\alpha/B_{a}}-1}{\eba^{\alpha}-1}\sum_{i=1}^{B_{a}}w_{i}\eba^{\alpha\left(B_{a}-i\right)/B_{a}}\\
 & =\frac{\eba^{\alpha/B_{a}}-1}{\eba^{\alpha}-1}\left(\sum_{i=0}^{B_{a}-1}w_{i}\eba^{\alpha\left(B_{a}-i\right)/B_{a}}+w_{B_{a}}-w_{0}\eba^{\alpha}\right)\\
 & =\frac{\eba^{\alpha/B_{a}}-1}{\eba^{\alpha}-1}\eba^{\alpha/B_{a}}\sum_{i=1}^{B_{a}}w_{i-1}\eba^{\alpha\left(B_{a}-i\right)/B_{a}}+\frac{\eba^{\alpha/B_{a}}-1}{\eba^{\alpha}-1}\left(w_{B_{a}}-w_{0}\eba^{\alpha}\right)\\
 & =\eba^{\alpha/B_{a}}\hat{\beta}_{a}^{(t-1)}+\frac{\eba^{\alpha/B_{a}}-1}{\eba^{\alpha}-1}\left(w_{B_{a}}-w_{0}\eba^{\alpha}\right).
\end{align*}
We set $z_{t}\coloneqq\alpha_{B}\left(w_{B_{a}}-\beta_{a}^{(t-1)}\right)$
which is feasible as the discounted value $w_{at}-\beta_{a}^{(t-1)}$
of the chosen advertiser $a$ may only be $\alpha_{B}$-times less
the maximum discounted value $w_{\aalgo t}-\beta_{\aalgo}^{(t-1)}$
due to the advantage of the predicted advertiser. This yields a dual
increase of
\begin{align*}
\Delta D & =B_{a}\left(\beta_{a}^{(t)}-\beta_{a}^{(t-1)}\right)+z_{t}\\
 & =B_{a}\left(\beta_{a}^{(t)}-\beta_{a}^{(t-1)}\right)+\alpha_{B}\left(w_{B_{a}}-\beta_{a}^{(t-1)}\right)\\
 & \le B_{a}\left(\beta_{a}^{(t)}-\hat{\beta}_{a}^{(t-1)}\right)+\alpha_{B}\left(w_{B_{a}}-\hat{\beta}_{a}^{(t-1)}\right)\\
 & =B_{a}\left(\left(\eba^{\alpha/B_{a}}-1\right)\hat{\beta}_{a}^{(t-1)}+\frac{\eba^{\alpha/B_{a}}-1}{\eba^{\alpha}-1}\left(w_{B_{a}}-w_{0}\eba^{\alpha}\right)\right)+\alpha_{B}\left(w_{B_{a}}-\hat{\beta}_{a}^{(t-1)}\right)\\
 & =\alpha_{B_{a}}\hat{\beta}_{a}^{(t-1)}+\frac{\alpha_{B_{a}}}{\eba^{\alpha}-1}\left(w_{B_{a}}-w_{0}\eba^{\alpha}\right)+\alpha_{B}\left(w_{B_{a}}-\hat{\beta}_{a}^{(t-1)}\right)\\
 & =\alpha_{B_{a}}\frac{\eba^{\alpha}}{\eba^{\alpha}-1}\underbrace{\left(\hat{\beta}_{a}^{(t-1)}-w_{0}\right)}_{\ge0}+\frac{\alpha_{B_{a}}}{\eba^{\alpha}-1}\underbrace{\left(w_{B_{a}}-\hat{\beta}_{a}^{(t-1)}\right)}_{\ge0}+\alpha_{B}\left(w_{B_{a}}-\hat{\beta}_{a}^{(t-1)}\right)\\
 & \le\alpha_{B}\frac{\eb^{\alpha}}{\eb^{\alpha}-1}\left(\hat{\beta}_{a}^{(t-1)}-w_{0}\right)+\frac{\alpha_{B}}{\eb^{\alpha}-1}\left(w_{B_{a}}-\hat{\beta}_{a}^{(t-1)}\right)+\alpha_{B}\left(w_{B_{a}}-\hat{\beta}_{a}^{(t-1)}\right)\\
 & =\alpha_{B}\frac{\eb^{\alpha}}{\eb^{\alpha}-1}\left(\hat{\beta}_{a}^{(t-1)}-w_{0}\right)+\alpha_{B}\frac{\eb^{\alpha}}{\eb^{\alpha}-1}\left(w_{B_{a}}-\hat{\beta}_{a}^{(t-1)}\right)\\
 & =\alpha_{B}\frac{\eb^{\alpha}}{\eb^{\alpha}-1}\left(w_{B_{a}}-w_{0}\right)=B\frac{\eb^{\alpha/B}-1}{\eb^{\alpha}-1}\eb^{\alpha}\left(w_{B_{a}}-w_{0}\right)
\end{align*}
where the second inequality is due to $\alpha_{B}\ge\alpha_{B_{a}}$
and $\eb\le\eba$, as shown in Lemma \ref{lem:param-props}. 
\end{proof}

\subsection{\label{subsec:omitted-consistency} Proof of Theorem \ref{thm:main}
(Consistency)}

In the following, we upper bound $\EST$ using the comparison in Line
\ref{alg:selection-rule} of Algorithm \ref{alg:exp-avg}. 
\begin{lem}
\label{lem:prd} We have
\begin{align*}
\EST & \le\sum_{a}\Bigg(\left(B_{a}-\ell_{a}\right)\beta_{a}^{(T)}+\frac{1}{\alpha_{B}}\sum_{\t\in\X_{a}\setminus\P_{a}}\left(w_{a\t}-\beta_{a}^{(\t-1)}\right)+\sum_{\t\in\P_{a}\cap\X_{a}}w_{at}\Bigg)
\end{align*}
\end{lem}

\begin{proof}
We first split impressions $\t$ into two categories: Either the algorithm
followed the prediction and assigned $\t$ to $a^{(\t)}=\apred^{(\t)}$,
or the algorithm ignored the prediction and assigned $\t$ to $a^{(\t)}=\aalgo^{(\t)}\not=\apred^{(\t)}$.
In the latter case, due to the selection rule in Line \ref{alg:selection-rule}
of Algorithm \ref{alg:exp-avg},
\[
\alpha_{B}\Big(w_{\apred^{(\t)}\t}-\beta_{\apred^{(\t)}}^{(\t-1)}\Big)\le w_{\aalgo^{(\t)}\t}-\beta_{\aalgo^{(\t)}}^{(\t-1)}.
\]
In symbols,
\begin{align*}
\EST & =\sum_{a}\Bigg(\sum_{\t\in\P_{a}\setminus\X_{a}}w_{a\t}+\sum_{\t\in\P_{a}\cap\X_{a}}w_{a\t}\Bigg)\\
 & \le\sum_{a}\Bigg(\sum_{\t\in\P_{a}\setminus\X_{a}}\left(\beta_{a}^{(\t-1)}+\frac{1}{\alpha_{B}}\left(w_{\aalgo^{(\t)},\t}-\beta_{\aalgo^{(\t)}}^{(\t-1)}\right)\right)+\sum_{\t\in\P_{a}\cap\X_{a}}w_{a\t}\Bigg)\\
 & =\sum_{a}\Bigg(\underbrace{\sum_{\t\in\P_{a}\setminus\X_{a}}\beta_{a}^{(\t-1)}}_{(\dagger)}+\frac{1}{\alpha_{B}}\sum_{\t\in\X_{a}\setminus\P_{a}}\left(w_{a\t}-\beta_{a}^{(\t-1)}\right)+\sum_{\t\in\P_{a}\cap\X_{a}}w_{at}\Bigg)
\end{align*}
where the last equality holds because $\left\{ \P_{a}\right\} _{a}$
and $\left\{ \X_{a}\right\} _{a}$ are both partitioning the set of
all impressions due to the introduction of the dummy advertiser.
For $(\dagger)$, we use that $\beta_{a}$ can only increase in each
round and bound
\[
\sum_{t\in\P_{a}\setminus\X_{a}}\beta_{a}^{(t-1)}\le\left(B_{a}-\ell_{a}\right)\beta_{a}^{(T)}.
\]
\end{proof}

For the remainder of this section, we consider a fixed advertiser
$a$. Let us denote with $\t_{\i}$ the $\i$-th impression allocated
to $a$. Let
\[
(\star)=\frac{1}{\alpha_{B}}\sum_{\t\in\X_{a}\setminus\P_{a}}\left(w_{a\t}-\beta_{a}^{(\t-1)}\right)+\sum_{\t\in\P_{a}\cap\X_{a}}w_{at}
\]
as part of the the bound on $\EST$ in Lemma \ref{lem:prd}.

In order to understand this bound, we make some useful observations
in the following lemma to simplify the analysis. The key idea is that
we may assume that impressions in $\X_{a}$ are ordered to be non-decreasing.
In particular, we need to argue that the sum $\sum_{t\in\X_{a}\setminus\P_{a}}\beta_{a}^{(t-1)}$
can only decrease (as this term is negated in $(\star)$) when impressions
in $\X_{a}$ are ordered to be non-decreasing: Intuitively, each $\beta_{a}^{(t-1)}$
depends only on the $B_{a}$ most valuable impressions assigned before
impression $t$, no matter the order in which $\X_{a}^{(t-1)}$ arrived.
We can thus minimize each $\beta_{a}^{(t-1)}$ if the impressions
allocated prior to $t$ are the impressions of smallest value. To
simultaneously minimize each $\beta_{a}^{(t)}$ in the sum, we order
the impressions in $\X_{a}$ to have non-decreasing value. We prove
this simplification formally in the following lemma.
\begin{lem}
\label{lem:wlog} Without loss of generality, we may assume that $\P_{a}\cap\X_{a}$
are the most valuable impressions in $\X_{a}$ and that impressions
in $\X_{a}$ arrive such that their values are non-decreasing.
\end{lem}

\begin{proof}
We may assume that the impressions in $\P_{a}\cap\X_{a}$ are the
most valuable impressions in $\X_{a}$: this can only increase the
value of $\P_{a}$ but leaves $\S_{a}$ unaffected, as $\S_{a}$ are
by design the $B_{a}$ most valuable impressions in $\X_{a}$. All
impressions in $(\star)$ are from $\X_{a}$, so reordering impressions
only affects $(\star)$. Specifically, we can show that the sum $\sum_{\t\in\X_{a}\setminus\P_{a}}\beta_{a}^{(\t-1)}$
in $(\star)$ is minimized if the values in $\X_{a}$ are ordered
to be non-decreasing.  Assume to the contrary that the $\i$-th impression
added to $a$ is the last that is in order. That is $\wi 1\le\wi 2\le\cdots\le\wi{\i}$
and there exists a $\j\le\i$ such that $\wi{\j-1}\le\wi{\i+1}<\wi{\j}$.
Moving $\t_{\i+1}$ ahead to its ranked position within the first
$\i$ impressions allocated to $a$ changes the ordering as follows
(the first and second row show the impression values before and after
changing the position of $\t_{\i+1}$, respectively):
\[
\arraycolsep=1pt\begin{array}{ccccccccccccc}
\wi 1 & \le & \cdots & \le & \wi{\j-1} & \le & \wi{\j} & \le & \wi{\j+1} & \le & \cdots & \le & \wi{\i}\\
\wi 1 & \le & \cdots & \le & \wi{\j-1} & \le & \wi{\i+1} & < & \wi{\j} & \le & \cdots & \le & \wi{\i-1}
\end{array}
\]
Note that each position decreases in value, even strictly at the $\j$-th
position. As such, the exponential average $\beta_{a}^{(\t-1)}$ decreases
as well for $\t<\t_{\i+1}$; it remains constant for $\t\ge\t_{\i+1}$
as it only depends on the $B_{a}$ most valuable impressions assigned
up to $t$ which remain the same.  We can thus simultaneously minimize
$\beta_{a}^{(\t)}$ for each $\t$ by putting $\X_{a}$ in non-decreasing
order. This reordering does not affect $\beta_{a}^{(T)}$ or the other
terms in $(\star)$, so we may indeed assume that values are non-decreasing. 
\end{proof}

In light of Lemma \ref{lem:wlog}, we can write $(\star)$ as follows.
\begin{lem}
\label{lem:star} We have
\[
(\star)=\frac{1}{\alpha_{B}}\sum_{\i=1}^{I_{a}-\ell_{a}}\left(\wi{\i}-\betai{\i-1}\right)+\sum_{\i=I_{a}-\ell_{a}+1}^{I_{a}}\wi{\i}.
\]
\end{lem}

\begin{proof}
Impression values are non-decreasing due to Lemma \ref{lem:wlog},
so $\wi{\i}$ is the $\i$-th least valuable impression in $\X_{a}$.
The impressions $\left\{ I_{a}-\ell_{a}+1,\dots,I_{a}\right\} =\X_{a}\cap\P_{a}$
are thus the most valuable. We can now write $(\star)$ as
\[
\frac{1}{\alpha_{B}}\sum_{\t\in\X_{a}\setminus\P_{a}}\left(w_{a\t}-\beta_{a}^{(\t-1)}\right)+\sum_{\t\in\P_{a}\cap\X_{a}}w_{a\t}=\frac{1}{\alpha_{B}}\sum_{\i=1}^{I_{a}-\ell_{a}}\left(\wi{\i}-\betai{\i-1}\right)+\sum_{\i=I_{a}-\ell_{a}+1}^{I_{a}}\wi{\i}
\]
where $\beta_{a}^{(t_{i}-1)}=\betai{i-1}$ as there was no change
to the dual variable of advertiser $a$ since no impression in $\left\{ \t_{\i-1}+1,\dots,\t_{\i}-1\right\} $
was allocated to $a$.
\end{proof}

Combining Lemmas \ref{lem:prd} and \ref{lem:star}, we obtain:
\begin{lem}
\label{lem:prd-all} $\EST\le\sum_{a}\EST_{a}$ where
\[
\EST_{a}\coloneqq\frac{1}{\alpha_{B}}\sum_{\i=1}^{I_{a}-\ell_{a}}\left(\wi{\i}-\betai{\i-1}\right)+\sum_{\i=I_{a}-\ell_{a}+1}^{I_{a}}\wi{\i}+\left(I_{a}-\ell_{a}\right)\beta_{a}^{(T)}.
\]
\end{lem}

In the following, we use the non-decreasing ordering of impressions
in $\X_{a}$ to compute $\betai{i-1}$ and bound $\EST_{a}$ with
a linear combination of values $\wi{\i}$. Consider the $\j$-th impression
$\t_{\j}$ allocated to $a$. Since we assume that impression values
are non-decreasing, we know that $\t_{\j}$ becomes the most valuable
impression right after it is allocated. After the allocation of the
$(\j+1)$-th impression to $a$, it becomes the second most valuable
impression, and so forth, until it is disposed after the allocation
of the $(\j+B_{a})$-th impression. The value $\wi{\j}$ therefore
appears alongside each coefficient in the convex combination that
defines $\beta_{a}^{(t_{i-1})}$ for $\i\in\left\{ j+1,\dots,j+B_{a}\right\} $.
Expanding each $\betai{i-1}$ in the sum $\sum_{\i=1}^{I_{a}-\ell_{a}}\betai{\i-1}$
in $\EST_{a}$, we thus observe that the coefficients of values $\wi j$
for $\j\le I_{a}-\ell_{a}-B_{a}$ sum up to $1$. We use this fact
to cancel out most of the values in $\sum_{\i=1}^{I_{a}-\ell_{a}}\wi{\i}$.
What remains are only the values $\wi{\i}$ for $\i\in\left\{ I_{a}-\ell_{a}-B_{a}+1,\dots,I_{a}-\ell_{a}\right\} $.
For $\i\in\left\{ I_{a}-\ell_{a}-B_{a}+1,\dots,I_{a}-B_{a}\right\} $,
we bound $\wi{\i}$ by $\wi{I_{a}-B_{a}}$ which is really the best
we can hope for. Formally, we show:
\begin{lem}
\label{lem:1} We have
\[
\EST_{a}\le\hspace{-10pt}\sum_{\i=I_{a}-B_{a}+1}^{I_{a}-\ell_{a}}\phi_{\i}\wi{\i}+\hspace{-10pt}\sum_{\i=I_{a}-\ell_{a}+1}^{I_{a}}\psi_{\i}\wi{\i}+\wi{I_{a}-B_{a}}\Omega_{a}
\]
with coefficients
\begin{align*}
\phi_{\i} & \coloneqq\left(B_{a}-\ell_{a}\right)\norm\eba^{\alpha\left(I_{a}-\i\right)/B_{a}}+\frac{1}{\alpha_{B}}\frac{\eba^{\alpha}-\eba^{\alpha\left(I_{a}-\ell_{a}-\i\right)/B_{a}}}{\eba^{\alpha}-1}\\
\psi_{\i} & \coloneqq1+\left(B_{a}-\ell_{a}\right)\norm\eba^{\alpha\left(I_{a}-\i\right)/B_{a}}\\
\Omega_{a} & \coloneqq\frac{1}{\alpha_{B}}\frac{1}{\eba^{\alpha}-1}\left(\ell_{a}\eba^{\alpha}-\frac{\eba^{\alpha}-\eba^{\alpha\left(B_{a}-\ell_{a}\right)/B_{a}}}{\eba^{\alpha/B_{a}}-1}\right).
\end{align*}
\end{lem}

\begin{proof}
We start by rewriting the terms in $\EST_{a}$ individually. Since
we assume that the values are non-decreasing, we can express $\betai{\i-1}$
as the exponential average of values $\wi{\i-B_{a}},\wi{\i-B_{a}+1},\dots,\wi{\i-1}$
of the last $B_{a}$ impressions (for simplicity, we set $\wi{\j}=0$
for $\j\le0$). Summing over multiple iterations, we thus obtain for
the sum over the dual variables that
\begin{align*}
\sum_{\i=1}^{I_{a}-\ell_{a}}\betai{\i-1} & =\norm\sum_{\i=1}^{I_{a}-\ell_{a}}\sum_{\j=\i-B_{a}}^{\i-1}\wi j\eba^{\alpha\left(\i-\j-1\right)/B_{a}}\\
 & =\norm\sum_{\j=1}^{I_{a}-\ell_{a}}\wi{\j}\sum_{\i=\j+1}^{\min\left\{ \j+B_{a},I_{a}-\ell_{a}\right\} }\eba^{\alpha\left(\i-\j-1\right)/B_{a}}\\
 & =\norm\sum_{\j=1}^{I_{a}-\ell_{a}}\wi{\j}\sum_{\i=1}^{\min\left\{ B_{a},I_{a}-\ell_{a}-\j\right\} }\eba^{\alpha\left(\i-1\right)/B_{a}}\\
 & =\norm\sum_{\j=1}^{I_{a}-B_{a}-\ell_{a}}\wi{\j}\sum_{\i=1}^{I_{a}}\eba^{\alpha\left(\i-1\right)/B_{a}}\\
 & \quad+\norm\sum_{\j=I_{a}-B_{a}-\ell_{a}+1}^{I_{a}-\ell_{a}}w_{a}\sum_{\i=1}^{I_{a}-\ell_{a}-\j}\eba^{\alpha\left(\i-1\right)/B_{a}}\\
 & =\sum_{\i=1}^{I_{a}-B_{a}-\ell_{a}}\wi{\i}+\frac{1}{\eba^{\alpha}-1}\sum_{\i=I_{a}-B_{a}-\ell_{a}+1}^{I_{a}-\ell_{a}}\wi{\i}\left(\eba^{\alpha\left(I_{a}-\ell_{a}-\i\right)/B_{a}}-1\right).
\end{align*}
where for the last equality, we use that the two inner sums are geometric.
We can use this expression to cancel out most of the terms of the
first sum in $\EST_{a}$:
\begin{align}
 & \sum_{\i=1}^{I_{a}-\ell_{a}}\left(\wi{\i}-\betai{\i-1}\right)\nonumber \\
 & =\sum_{\i=1}^{I_{a}-\ell_{a}}\wi{\i}-\sum_{\i=1}^{I_{a}-B_{a}-\ell_{a}}\wi{\i}-\frac{1}{\alpha_{B}\left(\eba^{\alpha}-1\right)}\sum_{\i=I_{a}-B_{a}-\ell_{a}+1}^{I_{a}-\ell_{a}}\wi{\i}\left(\eba^{\alpha\left(I_{a}-\ell_{a}-\i\right)/B_{a}}-1\right)\nonumber \\
 & =\sum_{\i=I_{a}-B_{a}-\ell_{a}+1}^{I_{a}-\ell_{a}}\wi{\i}\left(1-\frac{\eba^{\alpha\left(I_{a}-\ell_{a}-\i\right)/B_{a}}-1}{\eba^{\alpha}-1}\right)\nonumber \\
 & =\sum_{\i=I_{a}-B_{a}-\ell_{a}+1}^{I_{a}-\ell_{a}}\wi{\i}\frac{\eba^{\alpha}-\eba^{\alpha\left(I_{a}-\ell_{a}-\i\right)/B_{a}}}{\eba^{\alpha}-1}\nonumber \\
 & =\sum_{\i=I_{a}-B_{a}+1}^{I_{a}-\ell_{a}}\wi{\i}\frac{\eba^{\alpha}-\eba^{\alpha\left(I_{a}-\ell_{a}-\i\right)/B_{a}}}{\eba^{\alpha}-1}+\sum_{\i=I_{a}-B_{a}-\ell_{a}+1}^{I_{a}-B_{a}}\wi{\i}\frac{\eba^{\alpha}-\eba^{\alpha\left(I_{a}-\ell_{a}-\i\right)/B_{a}}}{\eba^{\alpha}-1}.\label{eq:6}
\end{align}
 We use that $\wi{\i}\le\wi{I_{a}-B_{a}}$ for all $i\le I_{a}-B_{a}$
to upper bound the second sum, divided by $\alpha_{B}$, in (\ref{eq:6})
to
\begin{align}
 & \frac{1}{\alpha_{B}}\sum_{\i=I_{a}-B_{a}-\ell_{a}+1}^{I_{a}-B_{a}}\wi{\i}\frac{\eba^{\alpha}-\eba^{\alpha\left(I_{a}-\ell_{a}-\i\right)/B_{a}}}{\eba^{\alpha}-1}\nonumber \\
 & \le\wi{I_{a}-B_{a}}\frac{1}{\alpha_{B}}\sum_{\i=I_{a}-B_{a}-\ell_{a}+1}^{I_{a}-B_{a}}\frac{\eba^{\alpha}-\eba^{\alpha\left(I_{a}-\ell_{a}-\i\right)/B_{a}}}{\eba^{\alpha}-1}\\
 & =\wi{I_{a}-B_{a}}\frac{1}{\alpha_{B}}\frac{1}{\eba^{\alpha}-1}\left(\ell_{a}\eba^{\alpha}-\sum_{\i=B_{a}-\ell_{a}}^{B_{a}-1}\eba^{\alpha\i/B_{a}}\right)\nonumber \\
 & =\wi{I_{a}-B_{a}}\underbrace{\frac{1}{\alpha_{B}}\frac{1}{\eba^{\alpha}-1}\left(\ell_{a}\eba^{\alpha}-\frac{\eba^{\alpha}-\eba^{\alpha\left(B_{a}-\ell_{a}\right)/B_{a}}}{\eba^{\alpha/B_{a}}-1}\right)}_{=\Omega_{a}}.\label{eq:14}
\end{align}
Furthermore, by definition of $\beta_{a}^{(T)}=\betai{I_{a}}$,
\begin{equation}
\left(B_{a}-\ell_{a}\right)\beta_{a}^{(T)}=\left(B_{a}-\ell_{a}\right)\norm\sum_{\i=I_{a}-B_{a}+1}^{I_{a}}\wi{\i}\eba^{\alpha\left(I_{a}-i\right)/B_{a}}.\label{eq:15}
\end{equation}
We combine (\ref{eq:6}), (\ref{eq:14}), and (\ref{eq:15}) and group
terms to obtain the desired bound
\begin{align*}
\EST_{a} & \le\sum_{i=I_{a}-\ell_{a}+1}^{I_{a}}\wi{\i}+\frac{1}{\alpha_{B}}\sum_{i=I_{a}-B_{a}+1}^{I_{a}-\ell_{a}}\wi{\i}\frac{\eba^{\alpha}-\eba^{\alpha\left(I_{a}-\ell_{a}-\i\right)/B_{a}}}{\eba^{\alpha}-1}+\wi{I_{a}-B_{a}}\Omega_{a}\\
 & \qquad+\left(B_{a}-\ell_{a}\right)\norm\sum_{\i=I_{a}-B_{a}+1}^{I_{a}}\wi{\i}\eba^{\alpha\left(I_{a}-\i\right)/B_{a}}\\
 & =\sum_{i=I_{a}-B_{a}+1}^{I_{a}-\ell_{a}}\wi{\i}\underbrace{\left(\left(B_{a}-\ell_{a}\right)\norm\eba^{\alpha\left(I_{a}-i\right)/B_{a}}+\frac{1}{\alpha_{B}}\frac{\eba^{\alpha}-\eba^{\alpha\left(I_{a}-\ell_{a}-\i\right)/B_{a}}}{\eba^{\alpha}-1}\right)}_{=\phi_{i}}\\
 & \qquad+\sum_{i=I_{a}-\ell_{a}+1}^{I_{a}}\wi{\i}\underbrace{\left(1+\left(B_{a}-\ell_{a}\right)\norm\eba^{\alpha\left(I_{a}-i\right)/B_{a}}\right)}_{=\psi_{i}}+w_{I_{a}-B_{a}}\Omega_{a}
\end{align*}
\end{proof}

We can express $\ALG$ analogously:
\begin{lem}
\label{lem:algfac} We have $\ALG=\sum_{a}\ALG_{a}$ where
\[
\ALG_{a}\coloneqq\sum_{\i=I_{a}-B_{a}+1}^{I_{a}}w_{at_{\i}}.
\]
\end{lem}

\begin{proof}
We have $\ALG=\sum_{a}\sum_{t\in\S_{a}}w_{at}$. As we always dispose
of the least valuable impression in Algorithm \ref{alg:exp-avg},
$\S_{a}$ are the $B_{a}$ most valuable impressions in $\X_{a}$.
Due to Lemma \ref{lem:wlog}, these are $\S_{a}=\left\{ I_{a}-B_{a}+1,\dots,I_{a}\right\} $
and hence $\sum_{t\in\S_{a}}w_{at}=\sum_{i=I_{a}-B_{a}+1}^{I_{a}}\wi{\i}=\ALG_{a}$.
\end{proof}

We upper bound the ratio $\EST/\ALG$ by $\max_{a}\EST_{a}/\ALG_{a}$.
To this end, we fix an advertiser $a$ and upper bound the ratio $\EST_{a}/\ALG_{a}$.
Recall from Lemmas \ref{lem:1} and \ref{lem:algfac} that we can
express $\EST_{a}$ and $\ALG_{a}$ as linear combination over impression
values. We could obtain a natural upper bound by comparing impression
value coefficients. However, in the following lemma, we show how to
use the non-decreasing ordering due to Lemma \ref{lem:wlog} to obtain
a tighter bound. 

We define
\[
\Phi_{a}\coloneqq\sum_{\i=I_{a}-B_{a}+1}^{I_{a}-\ell_{a}}\phi_{\i}\qquad\textrm{and}\qquad\Psi_{a}\coloneqq\sum_{\i=I_{a}-\ell_{a}+1}^{I_{a}}\psi_{\i}
\]
as the total factor mass on values $\wi{\i}$ for $\phi_{i}$ and
$\psi_{i}$, respectively. Let $\tau_{a}\coloneqq\left(\Phi_{a}+\Psi_{a}+\Omega_{a}\right)/B_{a}$
be the average factor. Recall that
\begin{align}
\EST_{a} & \le\sum_{\i=I_{a}-B_{a}+1}^{I_{a}-\ell_{a}}\phi_{\i}\wi{\i}+\sum_{\i=I_{a}-\ell_{a}+1}^{I_{a}}\psi_{\i}\wi{\i}+\wi{I_{a}-B_{a}}\Omega_{a}\label{eq:prd-1}\\
\ALG_{a} & =\sum_{\i=I_{a}-B_{a}+1}^{I_{a}}w_{at_{\i}}.\nonumber 
\end{align}
In the following lemma, we use~that $\wi{\i}\le\wi{\j}$ for $i\le j$
due to Lemma \ref{lem:wlog}, to further upper bound the RHS of \ref{eq:prd-1}
by a linear combination of the values, where we move mass from coefficients
on $\wi{\i}$ to coefficients on $\wi{\j}$. Additionally, we move
mass from $\Omega_{a}$ to coefficients $\phi_{i}$ for $\i\in\left\{ I_{a}-B_{a}+1,\dots,I_{a}-\ell_{a}\right\} $
and from $\phi_{\i}$ to $\psi_{\j}$ for $\j\in\left\{ I_{a}-\ell_{a}+1,\dots,I_{a}\right\} $.
In the best case, we are able to redistribute mass equally across
all values, in which case the consistency is given as the average
factor $\tau_{a}$. Otherwise, the factors on the largest values dominate,
giving us a consistency of $\Psi_{a}/\ell_{a}$. 
\begin{lem}
\label{lem:com} We have
\[
\frac{\EST_{a}}{\ALG_{a}}\le\begin{cases}
\max\left\{ \tau_{a},\frac{\Psi_{a}}{\ell_{a}}\right\}  & \text{if }\ell_{a}>0\\
\tau_{a} & \text{otherwise}
\end{cases}
\]
where
\[
\tau_{a}=1+\frac{1}{\eba^{\alpha}-1}\frac{1}{\alpha_{B}}\left(\eba^{\alpha}-\frac{\eba^{\alpha}-1}{\alpha_{B_{a}}}\right)
\]
and
\[
\frac{\Psi_{a}}{\ell_{a}}=1+\left(\frac{B_{a}}{\ell_{a}}-1\right)\frac{\eba^{\alpha\ell_{a}/B_{a}}-1}{\eba^{\alpha}-1}.
\]
\end{lem}

\begin{proof}
We calculate $\tau_{a}$ and $\Psi_{a}/\ell_{a}$ separately in Lemma
\ref{lem:avgfac} below. Our main goal is to distribute mass from
the factors $\phi_{i}$, $\psi_{i}$, and from $\Omega_{a}$ equally
to the values $w_{I_{a}-B_{a}+1},\dots,w_{I_{a}}$. We begin by taking
a closer look at the factors $\phi_{i}$ and $\psi_{i}$. First, note
that $\psi_{i}$ is always decreasing in $i$ as 
\[
\psi_{i}=1+\underbrace{\left(B_{a}-\ell_{a}\right)}_{\ge0}\underbrace{\norm}_{\ge0}\eba^{\alpha\left(I_{a}-i\right)/B_{a}}.
\]
We can therefore bound the linear combination over values in $\left\{ I_{a}-\ell_{a}+1,\dots,I_{a}\right\} $
using the average value $\bar{w}_{\Psi}\coloneqq\frac{1}{\ell_{a}}\sum_{i=I_{a}-\ell_{a}+1}^{I_{a}}\wi{\i}$
as
\begin{equation}
\sum_{i=I_{a}-\ell_{a}+1}^{I_{a}}\wi{\i}\psi_{i}\le\sum_{i=I_{a}-\ell_{a}+1}^{I_{a}}\bar{w}_{\Psi}\psi_{i}=\bar{w}_{\Psi}\Psi_{a}.\label{eq:16-1}
\end{equation}
 However, $\phi_{i}$ is not always decreasing which can be seen
by rearranging
\begin{align*}
\phi_{i} & =\left(B_{a}-\ell_{a}\right)\norm e^{\alpha\left(I_{a}-i\right)/B_{a}}+\frac{1}{\alpha_{B}}\frac{\eba^{\alpha}-\eba^{\alpha\left(I_{a}-\ell_{a}-i\right)/B_{a}}}{\eba^{\alpha}-1}\\
 & =\frac{1}{\eba^{\alpha}-1}\left(\left(B_{a}-\ell_{a}\right)\left(\eba^{\alpha/B_{a}}-1\right)-\frac{1}{\alpha_{B}}\eba^{-\alpha\ell_{a}/B_{a}}\right)\eba^{\alpha\left(I_{a}-i\right)/B_{a}}+\frac{1}{\alpha_{B}}\frac{\eba^{\alpha}}{\eba^{\alpha}-1}.
\end{align*}
We observe that $\phi_{i}$ is decreasing if $\left(B_{a}-\ell_{a}\right)\left(\eba^{\alpha/B_{a}}-1\right)$
is at least $\frac{1}{\alpha_{B}}\eba^{-\alpha\ell_{a}/B_{a}}$, and
we analyze two cases based on the relationship of both terms.

Let us first assume that $\left(B_{a}-\ell_{a}\right)\left(\eba^{\alpha/B_{a}}-1\right)\ge\frac{1}{\alpha_{B}}\eba^{-\alpha\ell_{a}/B_{a}}$
such that $\phi_{i}$ is decreasing in $i$ which helps us to bound
the linear combinations in Lemma \ref{lem:1} over $\left\{ I_{a}-B_{a}+1,\dots,I_{a}-\ell_{a}\right\} $
and $\left\{ I_{a}-\ell_{a}+1,\dots,T\right\} $ by the average values
$\bar{w}_{\Phi}\coloneqq\frac{1}{B_{a}-\ell_{a}}\sum_{\i=I_{a}-B_{a}+1}^{I_{a}-\ell_{a}}\wi{\i}$
and $\bar{w}_{\Psi}$, respectively. We further use that $\wi{I_{a}-B_{a}}\le\bar{w}_{\Phi}$
to charge mass from $\Omega_{a}$ to $\Phi_{a}$ and obtain due to
(\ref{eq:16-1}) that
\begin{align}
 & \sum_{\i=I_{a}-B_{a}+1}^{I_{a}-\ell_{a}}\wi{\i}\phi_{i}+\sum_{\i=I_{a}-\ell_{a}+1}^{I_{a}}\wi{\i}\psi_{i}+\wi{I_{a}-B_{a}}\Omega_{a}\nonumber \\
 & \le\bar{w}_{\Phi}\Phi_{a}+\bar{w}_{\Psi}\Psi_{a}+\wi{I_{a}-B_{a}}\Omega_{a}\\
 & \le\bar{w}_{\Phi}\left(\Phi_{a}+\Omega_{a}\right)+\bar{w}_{\Psi}\Psi_{a}\nonumber \\
 & =\bar{w}_{\Phi}\left(B_{a}-\ell_{a}\right)\tau_{a}+\bar{w}_{\Phi}\left(\Phi_{a}+\Omega_{a}-\left(B_{a}-\ell_{a}\right)\tau_{a}\right)+\bar{w}_{\Psi}\Psi_{a}\nonumber \\
 & =\sum_{\i=I_{a}-B_{a}+1}^{I_{a}-\ell_{a}}\tau_{a}\wi{\i}+\bar{w}_{\Phi}\left(\Phi_{a}+\Omega_{a}-\left(B_{a}-\ell_{a}\right)\tau_{a}\right)+\bar{w}_{\Psi}\Psi_{a}\label{eq:9-1}
\end{align}
On the other hand, if $\left(B_{a}-\ell_{a}\right)\left(\eba^{\alpha/B_{a}}-1\right)\le\frac{1}{\alpha_{B}}\eba^{-\alpha\ell_{a}/B_{a}}$
 we can no longer bound the values over $\left\{ I_{a}-B_{a}+1,\dots,I_{a}-\ell_{a}\right\} $
by the average value $\bar{w}_{\Phi}$. However, each factor $\phi_{i}$
is less than $\tau_{a}$ which can be seen by rearranging
\begin{multline*}
\phi_{\i}=\frac{1}{\eba^{\alpha}-1}\left(\left(B_{a}-\ell_{a}\right)\left(\eba^{\alpha/B_{a}}-1\right)-\frac{1}{\alpha_{B}}\eba^{-\alpha\ell_{a}/B_{a}}\right)\eba^{\alpha\left(I_{a}-\i\right)/B_{a}}+\frac{1}{\alpha_{B}}\frac{\eba^{\alpha}}{\eba^{\alpha}-1}\\
\le1+\frac{1}{\eba^{\alpha}-1}\frac{1}{\alpha_{B}}\left(\eba^{\alpha}-\frac{\eba^{\alpha}-1}{\alpha_{B_{a}}}\right)=\tau_{a}
\end{multline*}
to the equivalent expression
\begin{multline*}
\underbrace{\left(\left(B_{a}-\ell_{a}\right)\left(\eba^{\alpha/B_{a}}-1\right)-\frac{1}{\alpha_{B}}\eba^{-\alpha\ell_{a}/B_{a}}\right)}_{\le0}\underbrace{\eba^{\alpha\left(I_{a}-\i\right)/B_{a}}}_{\ge0}\\
\le\eba^{\alpha}-1-\frac{1}{\alpha_{B}}\frac{\eba^{\alpha}-1}{\alpha_{B_{a}}}=\underbrace{\left(1-\frac{1}{\alpha_{B}\cdot a_{B_{a}}}\right)}_{\ge0}\underbrace{\left(\eba^{\alpha}-1\right)}_{\ge0}
\end{multline*}
which is true since the LHS is $\le0$ and the RHS $\ge0$. We can
thus charge $\tau_{a}-\phi_{i}$ of mass from $\Omega_{a}$ to the
coefficients $\phi_{i}$ for each $i\in\left\{ I_{a}-B_{a}+1,\dots,I_{a}-\ell_{a}\right\} $
which yields
\begin{align}
 & \sum_{\i=I_{a}-B_{a}+1}^{I_{a}-\ell_{a}}\wi{\i}\phi_{i}+\sum_{\i=I_{a}-\ell_{a}+1}^{I_{a}}\wi{\i}\psi_{i}+\wi{I_{a}-B_{a}}\Omega_{a}\nonumber \\
 & \le\sum_{\i=I_{a}-B_{a}+1}^{I_{a}-\ell_{a}}\wi{\i}\phi_{i}+\bar{w}_{\Psi}\Psi_{a}+\wi{I_{a}-B_{a}}\Omega_{a}\nonumber \\
 & =\sum_{\i=I_{a}-B_{a}+1}^{I_{a}-\ell_{a}}\tau_{a}\wi{\i}-\sum_{\i=I_{a}-B_{a}+1}^{I_{a}-\ell_{a}}\wi{\i}\underbrace{\left(\tau_{a}-\phi_{i}\right)}_{\ge0}+\bar{w}_{\Psi}\Psi_{a}+\wi{I_{a}-B_{a}}\Omega_{a}\nonumber \\
 & \le\sum_{\i=I_{a}-B_{a}+1}^{I_{a}-\ell_{a}}\tau_{a}\wi{\i}-\sum_{\i=I_{a}-B_{a}+1}^{I_{a}-\ell_{a}}\wi{I_{a}-B_{a}}\left(\tau_{a}-\phi_{i}\right)+\bar{w}_{\Psi}\Psi_{a}+\wi{I_{a}-B_{a}}\Omega_{a}\nonumber \\
 & =\sum_{\i=I_{a}-B_{a}+1}^{I_{a}-\ell_{a}}\tau_{a}\wi{\i}+\wi{I_{a}-B_{a}}\left(\Phi_{a}+\Omega_{a}-\left(B_{a}-\ell_{a}\right)\tau_{a}\right)+\bar{w}_{\Psi}\Psi_{a}\label{eq:8-1}
\end{align}
In both cases (\ref{eq:9-1}) and (\ref{eq:8-1}), we have shown that
\begin{multline*}
\sum_{\i=I_{a}-B_{a}+1}^{I_{a}-\ell_{a}}\wi{\i}\phi_{i}+\sum_{\i=I_{a}-\ell_{a}+1}^{I_{a}}\wi{\i}\psi_{i}+\wi{I_{a}-B_{a}}\Omega_{a}\\
\le\sum_{\i=I_{a}-B_{a}+1}^{I_{a}-\ell_{a}}\tau_{a}\wi{\i}+v\left(\Phi_{a}+\Omega_{a}-\left(B_{a}-\ell_{a}\right)\tau_{a}\right)+\bar{w}_{\Psi}\Psi_{a}
\end{multline*}
for a $v\le\bar{w}_{\Psi}$. If $\ell_{a}>0$, we can use $v\le\bar{w}_{\Psi}$
to charge the remaining mass to $\Psi_{a}$ if the factors over $\left\{ I_{a}-\ell_{a}+1,\dots,T\right\} $
leave enough space. In symbols, this means
\begin{align*}
 & \sum_{\i=T-B_{a}+1}^{I_{a}-\ell_{a}}\tau_{a}\wi{\i}+v\left(\Phi_{a}+\Omega_{a}-\left(B_{a}-\ell_{a}\right)\tau_{a}\right)+\bar{w}_{\Psi}\Psi_{a}\\
 & \le\sum_{\i=T-B_{a}+1}^{I_{a}-\ell_{a}}\tau_{a}\wi{\i}+\bar{w}_{\Psi}\max\left\{ \Phi_{a}+\Omega_{a}-\left(B_{a}-\ell_{a}\right)\tau_{a},0\right\} +\bar{w}_{\Psi}\Psi_{a}\\
 & =\sum_{\i=T-B_{a}+1}^{I_{a}-\ell_{a}}\tau_{a}\wi{\i}+\bar{w}_{\Psi}\max\left\{ \Phi_{a}+\Psi_{a}+\Omega_{a}-\left(B_{a}-\ell_{a}\right)\tau_{a},\Psi_{a}\right\} \\
 & =\sum_{\i=T-B_{a}+1}^{I_{a}-\ell_{a}}\tau_{a}\wi{\i}+\bar{w}_{\Psi}\max\left\{ \ell_{a}\tau_{a},\Psi_{a}\right\} \\
 & \le\tau_{a}\sum_{\i=T-B_{a}+1}^{I_{a}-\ell_{a}}\wi{\i}+\max\left\{ \tau_{a},\frac{\Psi_{a}}{\ell_{a}}\right\} \sum_{\i=T-\ell_{a}+1}^{I_{a}}\wi{\i}\\
 & \le\max\left\{ \tau_{a},\frac{\Psi_{a}}{\ell_{a}}\right\} \sum_{\t\in S_{a}}w_{a\t}.
\end{align*}
If $\ell_{a}=0$, we have $\Psi_{a}=0$ and immediately obtain by
definition of $\tau_{a}$ that
\begin{align*}
\sum_{\i=I_{a}-B_{a}+1}^{I_{a}-\ell_{a}}\tau_{a}\wi{\i}+v\left(\Phi_{a}+\Omega_{a}-\left(B_{a}-\ell_{a}\right)\tau_{a}\right)+\bar{w}_{\Psi}\Psi_{a} & =\sum_{\i=I_{a}-B_{a}+1}^{I_{a}-\ell_{a}}\tau_{a}\wi{\i}.
\end{align*}
\end{proof}

\begin{lem}
\label{lem:avgfac} We have
\[
\frac{\Psi_{a}}{\ell_{a}}=1+\left(\frac{B_{a}}{\ell_{a}}-1\right)\frac{\eba^{\alpha\ell_{a}/B_{a}}-1}{\eba^{\alpha}-1}
\]
and
\[
\tau_{a}=1+\frac{1}{\eba^{\alpha}-1}\frac{1}{\alpha_{B}}\left(\eba^{\alpha}-\frac{\eba^{\alpha}-1}{\alpha_{B_{a}}}\right).
\]
\end{lem}

\begin{proof}
We compute
\begin{align*}
\Phi_{a} & =\sum_{i=T-B_{a}+1}^{T-\ell_{a}}\phi_{i}\\
 & =\sum_{i=T-B_{a}+1}^{T-\ell_{a}}\left(\frac{1}{\eba^{\alpha}-1}\left(\left(B_{a}-\ell_{a}\right)\left(\eba^{\alpha/B_{a}}-1\right)-\frac{1}{\alpha_{B}}\eba^{-\alpha\ell_{a}/B_{a}}\right)\eba^{\alpha\left(T-i\right)/B_{a}}+\frac{1}{\alpha_{B}}\frac{\eba^{\alpha}}{\eba^{\alpha}-1}\right)\\
 & =\frac{1}{\eba^{\alpha}-1}\left(\left(B_{a}-\ell_{a}\right)\left(\eba^{\alpha/B_{a}}-1\right)-\frac{1}{\alpha_{B}}\eba^{-\alpha\ell_{a}/B_{a}}\right)\frac{\eba^{\alpha}-\eba^{\alpha\ell_{a}/B_{a}}}{\eba^{\alpha/B_{a}}-1}+\left(B_{a}-\ell_{a}\right)\frac{1}{\alpha_{B}}\frac{\eba^{\alpha}}{\eba^{\alpha}-1}\\
 & =\frac{1}{\eba^{\alpha}-1}\left(B_{a}-\ell_{a}\right)\left(\eba^{\alpha}-\eba^{\alpha\ell_{a}/B_{a}}+\frac{1}{\alpha_{B}}\eba^{\alpha}\right)-\frac{1}{\eba^{\alpha}-1}\frac{1}{\alpha_{B}}\frac{\eba^{\alpha-\alpha\ell_{a}/B_{a}}-1}{\eba^{\alpha/B_{a}}-1}
\end{align*}
and
\begin{align*}
\Psi_{a} & =\sum_{i=T-\ell_{a}+1}^{T}\psi_{i}\\
 & =\sum_{i=T-\ell_{a}+1}^{T}\left(1+\left(B_{a}-\ell_{a}\right)\norm\eba^{\alpha\left(T-i\right)/B_{a}}\right)\\
 & =\ell_{a}+\left(B_{a}-\ell_{a}\right)\norm\frac{\eba^{\alpha\ell_{a}/B_{a}}-1}{\eba^{\alpha/B_{a}}-1}\\
 & =\ell_{a}+\left(B_{a}-\ell_{a}\right)\frac{\eba^{\alpha\ell_{a}/B_{a}}-1}{\eba^{\alpha}-1}.
\end{align*}
Summing up,
\begin{align*}
 & \Phi_{a}+\Psi_{a}+\Omega_{a}\\
 & =\frac{1}{e^{\alpha}-1}\left(B_{a}-\ell_{a}\right)\left(\eba^{\alpha}-\eba^{\alpha\ell_{a}/B_{a}}+\frac{1}{\alpha_{B}}\eba^{\alpha}\right)-\frac{1}{\eba^{\alpha}-1}\frac{1}{\alpha_{B}}\frac{\eba^{\alpha-\alpha\ell_{a}/B_{a}}-1}{\eba^{\alpha/B_{a}}-1}\\
 & \quad+\ell_{a}+\left(B_{a}-\ell_{a}\right)\frac{\eba^{\alpha\ell_{a}/B_{a}}-1}{\eba^{\alpha}-1}+\frac{1}{\eba^{\alpha}-1}\frac{1}{\alpha_{B}}\left(\ell_{a}\eba^{\alpha}-\frac{\eba^{\alpha}-\eba^{\alpha\left(B_{a}-\ell_{a}\right)/B_{a}}}{\eba^{\alpha/B_{a}}-1}\right)\\
 & =\frac{1}{\eba^{\alpha}-1}\left(B_{a}-\ell_{a}\right)\left(\eba^{\alpha}-1+\frac{1}{\alpha_{B}}\eba^{\alpha}\right)-\frac{1}{\eba^{\alpha}-1}\frac{1}{\alpha_{B}}\frac{\eba^{\alpha}-1}{\eba^{\alpha/B_{a}}-1}\\
 & \quad+\ell_{a}+\frac{1}{\eba^{\alpha}-1}\frac{1}{\alpha_{B}}\ell_{a}\eba^{\alpha}\\
 & =B_{a}+\frac{1}{\eba^{\alpha}-1}\frac{1}{\alpha_{B}}\left(B_{a}-\ell_{a}\right)\eba^{\alpha}-\frac{1}{\eba^{\alpha}-1}\frac{1}{\alpha_{B}}\frac{\eba^{\alpha}-1}{\eba^{\alpha/B_{a}}-1}+\frac{1}{\eba^{\alpha}-1}\frac{1}{\alpha_{B}}\ell_{a}\eba^{\alpha}\\
 & =B_{a}+\frac{1}{\eba^{\alpha}-1}\frac{1}{\alpha_{B}}B_{a}\eba^{\alpha}-\frac{1}{\eba^{\alpha}-1}\frac{1}{\alpha_{B}}\frac{\eba^{\alpha}-1}{\eba^{\alpha/B_{a}}-1}\\
 & =B_{a}+\frac{1}{\eba^{\alpha}-1}\frac{1}{\alpha_{B}}\left(B_{a}\eba^{\alpha}-\frac{\eba^{\alpha}-1}{\eba^{\alpha/B_{a}}-1}\right)
\end{align*}
which does no longer depend on $\ell_{a}$. Dividing $\Psi_{a}$ by
$\ell_{a}$ and $\Phi_{a}+\Psi_{a}+\Omega_{a}$ by $B_{a}$ yields
the result.
\end{proof}

Putting everything together, we have $\EST/\ALG\le\max_{a}\max\left\{ \tau_{a},\max_{\ell_{a}\in\left\{ 1,\dots,B_{a}\right\} }\Psi_{a}/\ell_{a}\right\} $
as $\tau_{a}$ does not depend on $\ell_{a}$. The reader can refer
back to Figure \ref{fig:cr-trade} for an illustration of this upper
bound. In the following lemma, we further analyze analytically $\max_{\ell_{a}\in\left\{ 1,\dots,B_{a}\right\} }\Psi_{a}/\ell_{a}$
and compare it with $\tau_{a}$ to obtain the upper bound:

\begin{lem}
\label{lem:bound} The consistency of Algorithm \ref{alg:exp-avg}
is given by
\[
\EST/\ALG\le\left(1+\frac{1}{\eb^{\alpha}-1}\max\left\{ \frac{1}{\alpha_{B}}\left(\eb^{\alpha}-\frac{\eb^{\alpha}-1}{\alpha_{B}}\right),\ln\left(\eb^{\alpha}\right)\right\} \right).
\]
\end{lem}

\begin{proof}
Due to Lemma \ref{lem:com}, it is sufficient to show
\[
1+\frac{1}{\eb^{\alpha}-1}\max\left\{ \frac{1}{\alpha_{B}}\left(\eb^{\alpha}-\frac{\eb^{\alpha}-1}{\alpha_{B}}\right),\ln\left(\eb^{\alpha}\right)\right\} \ge\begin{cases}
\max\left\{ \tau_{a},\Psi_{a}/\ell_{a}\right\}  & \text{if }\ell_{a}>0\\
\tau_{a} & \text{otherwise}.
\end{cases}
\]
By Lemma \ref{lem:avgfac}, we know for the first term in the maximum
that 
\[
\tau_{a}=1+\frac{1}{\eba^{\alpha}-1}\frac{1}{\alpha_{B}}\left(\eba^{\alpha}-\frac{\eba^{\alpha}-1}{\alpha_{B_{a}}}\right)
\]
This term is maximized for $B_{a}=B$ since
\begin{multline*}
1+\frac{1}{\eba^{\alpha}-1}\frac{1}{\alpha_{B}}\left(\eba^{\alpha}-\frac{\eba^{\alpha}-1}{B_{a}\left(\eba^{\alpha/B_{a}}-1\right)}\right)=1+\frac{1}{\alpha_{B}}\left(\frac{\eba^{\alpha}}{\eba^{\alpha}-1}-\frac{1}{\alpha_{B_{a}}}\right)\\
\le1+\frac{1}{\alpha_{B}}\left(\frac{\eb^{\alpha}}{\eb^{\alpha}-1}-\frac{1}{\alpha_{B}}\right)=1+\frac{1}{\eb^{\alpha}-1}\underbrace{\frac{1}{\alpha_{B}}\left(\eb^{\alpha}-\frac{\eb^{\alpha}-1}{\alpha_{B}}\right)}_{\eqqcolon p(\alpha)}
\end{multline*}
due to Lemma \ref{lem:param-props}. The lemma statement therefore
follows immediately if $\ell_{a}=0$. We may thus assume that $\ell_{a}>0$
and use Lemma \ref{lem:avgfac} to determine the second term in the
maximum as
\[
\frac{\Phi_{a}}{\ell_{a}}=1+\frac{1}{\eba^{\alpha}-1}\left(\frac{1}{x}-1\right)\left(\eba^{\alpha x}-1\right)
\]
where $x\eqqcolon\ell_{a}/B_{a}$. The second term behaves similarly
to the first as
\begin{align*}
\frac{\Phi_{a}}{\ell_{a}}=1+\frac{1}{\eba^{\alpha}-1}\left(\frac{1}{x}-1\right)\left(\eba^{\alpha x}-1\right) & \le1+\frac{1}{\eb^{\alpha}-1}\left(\frac{1}{x}-1\right)\left(\eb^{\alpha x}-1\right)
\end{align*}
since $\left(\eba^{\alpha x}-1\right)/\left(\eba^{\alpha}-1\right)\le\left(\eb^{\alpha x}-1\right)/\left(\eb^{\alpha}-1\right)$.
We define $g(\alpha,x)\coloneqq\left(\frac{1}{x}-1\right)\left(\eb^{\alpha x}-1\right)$
such that we can write
\[
\max\left\{ \frac{\Phi_{a}+\Psi_{a}+\Omega_{a}}{B_{a}},\frac{\Psi_{a}}{\ell_{a}}\right\} \le1+\frac{1}{\eb^{\alpha}-1}\max\left\{ p(\alpha),g(\alpha,x)\right\} .
\]
We want to remove the dependency on $x$ in $g$ by maximizing $g$
over $x\in[0,1]$ for a fixed $\alpha$. As $g(\alpha,x)$ is continuous,
it suffices to evaluate $g$ in both endpoints and find the stationary
points. We have
\begin{multline*}
g(\alpha,0)=\lim_{x\to0}\left(\frac{1}{x}-1\right)\left(\eb^{\alpha x}-1\right)=\lim_{x\to0}\frac{\left(1-x\right)\left(\eb^{\alpha x}-1\right)}{x}\\
=\lim_{x\to0}-\left(\eb^{\alpha x}-1\right)+\left(1-x\right)\ln(\eb^{\alpha})\eb^{\alpha x}=\ln(\eb^{\alpha})
\end{multline*}
by L'Hoptial. Further, $g(\alpha,1)=0$. Next, we find the stationary
points $x^{*}\in\left[0,1\right]$ as solutions to the equation
\[
\frac{\partial}{\partial x}g(\alpha,x^{*})=\ln(\eb^{\alpha})\left(\frac{1}{x^{*}}-1\right)\eb^{\alpha x^{*}}-\frac{\eb^{\alpha x^{*}}-1}{(x^{*})^{2}}=0
\]
which is equivalent to
\[
\eb^{\alpha x^{*}}-1=\ln(\eb^{\alpha})(x^{*})^{2}\left(\frac{1}{x^{*}}-1\right)\eb^{\alpha x^{*}}.
\]
There is no closed form solution for $x^{*}$, but we can replace
$\eb^{\alpha x}-1$ in $g$ with the RHS of the above. This yields
a new function
\[
h(\alpha,y)=\left(\frac{1}{y}-1\right)\ln(\eb^{\alpha})y^{2}\left(\frac{1}{y}-1\right)\eb^{\alpha y}=\ln(\eb^{\alpha})\left(1-y\right)^{2}\eb^{\alpha y}
\]
with $h(\alpha,x^{*})=g(\alpha,x^{*})$. We can thus maximize $h$
over $y\in\left[0,1\right]$ to obtain an upper bound on $g(x^{*})$.
Note that $h(\alpha,0)=\ln(\eb^{\alpha})=g(\alpha,0)$ and $h(\alpha,1)=0=g(\alpha,1)$.
To this end, let $y^{*}$ be such that
\[
\frac{\partial}{\partial y^{*}}h(\alpha,y^{*})=\ln(\eb^{\alpha})^{2}\left(1-y^{*}\right)^{2}\eb^{\alpha y^{*}}-2\ln(\eb^{\alpha})\left(1-y^{*}\right)\eb^{\alpha y^{*}}=0
\]
which is equivalent to $\ln(\eb^{\alpha})\left(1-y^{*}\right)-2=0$
or $y^{*}=1-\frac{2}{\ln(\eb^{\alpha})}$. We evaluate $h$ in $y^{*}$
and obtain
\[
h(\alpha,y^{*})=\alpha\ln(\eb)\left(\frac{2}{\alpha\ln(\eb)}\right)^{2}\eb^{\alpha-\frac{2}{\ln(\eb)}}=\frac{4}{\alpha\ln(\eb)e^{2}}\eb^{\alpha}.
\]
Note that $y^{*}\ge0\iff\alpha\ge2/\ln(\eb)$. Furthermore, $h^{*}(\alpha)$
always exceeds the endpoint $g(\alpha,0)$: We calculate
\begin{align*}
h^{*}(\alpha)\coloneqq h(\alpha,y^{*}) & =\frac{4}{\ln(\eb^{\alpha})e^{2}}\eb^{\alpha}\\
 & \ge\frac{4}{\ln(\eb^{\alpha})e^{2}}e^{2}\ln(\eb^{\alpha/2})^{2}\\
 & =\ln(\eb^{\alpha})
\end{align*}
where the inequality is due to $e^{z}\ge ez$ for $z=\ln\left(\eb^{\alpha/2}\right)\ge0$.
Therefore, for all $x\in\left[0,1\right]$,
\[
g(\alpha,x)\le\begin{cases}
\ln(\eb^{\alpha}) & \text{if }\alpha\le\frac{2}{\ln(\eb)}\\
h^{*}(\alpha) & \text{otherwise}.
\end{cases}
\]
We consider both intervals separately. Let us first consider the the
case when $\alpha\in\left[0,\frac{2}{\ln(\eb)}\right]$. If $B<\infty$,
there could be multiple intersection points between $p(\alpha)$ and
$\alpha\ln(\eb)$. However, the situation is easier if $B\to\infty$
as the intersection points given by
\[
p(\alpha)=\frac{1}{\alpha}\left(e^{\alpha}-\frac{e^{\alpha}-1}{\alpha}\right)=\alpha\iff\alpha e^{\alpha}-e^{\alpha}+1=\alpha^{3}
\]
are at $\alpha=1$ and $\alpha^{*}\approx1.79$, whereas $\alpha\ln(\eb)$
dominates $p(\alpha)$ between $1$ and $\alpha^{*}$. 

It remains to consider the case $\alpha\ge\frac{2}{\ln(\eba)}$. Again,
there can be many intersection points of $p(\alpha)$ with $\alpha\ln(\eba)$
and $h^{*}(\alpha)$. However, if $B\to\infty$, then $p(\alpha)$
already dominates $h^{*}(\alpha)$ for $\alpha>2$ which we can see
as follows. First,
\begin{align*}
h^{*}(\alpha)=\frac{4}{\alpha}e^{\alpha-2} & \le\frac{1}{\alpha}\left(e^{\alpha}-\frac{e^{\alpha}-1}{\alpha}\right)=p(\alpha)\\
\iff4e^{-2} & \le1-\frac{1-e^{-\alpha}}{\alpha}.
\end{align*}
We can see that $\frac{1-e^{-\alpha}}{\alpha}$ is decreasing in $\alpha$
as 
\[
\frac{\partial}{\partial\alpha}\frac{1-e^{-\alpha}}{\alpha}=\frac{e^{-\alpha}\left(\alpha-e^{\alpha}+1\right)}{\alpha^{2}}\le0
\]
which holds as $1+\alpha\le e^{\alpha}$. Finally, we check that $h^{*}(2)=2\le2.10\approx p(2)$.
\end{proof}

\section{Generalized Assignment Problem}

\label{sec:gap} The generalized assignment problem (GAP) is a generalization
of Display Ads where impressions $t$ can take up any size $u_{at}$
in the budget constraint of advertiser $a$. This formulation encompasses
both Display Ads and Ad Words, and we empirically compare it to the
Ad Words algorithm with predictions due to \citet{mahdian07} in Section
\ref{subsec:adwords-gap}. For simplicity of presentation, we assume
that budgets are all $1$ and instead, $u_{at}\to0$. However, as
before it is possible to adapt the algorithm to work with large sizes
$u_{at}$. We state the LP below.

\noindent\fbox{\begin{minipage}[t]{1\columnwidth - 2\fboxsep - 2\fboxrule}%
\begin{center}
\begin{minipage}[t]{0.3\columnwidth}%
\begin{align*}
\text{} & \text{GAP Primal}\\
\max & \sum_{a,t}w_{at}x_{at}\\
\forall a\colon & \sum_{t}u_{at}x_{at}\leq1\\
\forall t\colon & \sum_{a}x_{at}\leq1
\end{align*}
\end{minipage}%
\begin{minipage}[t]{0.3\columnwidth}%
\begin{align*}
 & \text{GAP Dual}\\
\min & \sum_{a}\beta_{a}+\sum_{t}z_{t}\\
\forall a,t\colon & z_{t}\ge w_{at}-u_{at}\beta_{a}
\end{align*}
\end{minipage}\medskip{}
\par\end{center}%
\end{minipage}}

\medskip{}
Algorithm \ref{alg:exp-avg-1} is a generalization of Algorithm \ref{alg:exp-avg}
to GAP. An immediate difference is that the discounted gain $w_{at}-u_{at}\beta_{a}$
respects the impression size $u_{at}$ in accordance with the changed
dual constraint. We still follow the predicted advertiser if its discounted
gain still is a sufficiently high fraction of the maximum discounted
gain. However, we might now have to remove multiple impressions with
least value-size ratio to accommodate the new impression. The update
for $\beta_{a}$ also differs and is based on value-size ratios of
impressions allocated to $a$: For a fixed advertiser $a$ let $U_{a}=\sum_{t\in\X_{a}}u_{at}$,
be the total size of all impressions ever allocated to $a$. For any
$x\in(0,U_{a}]$ define $\frac{w_{x}}{u_{x}}$ as the minimal ratio
such that
\begin{equation}
\sum_{t\in\X_{a}:\frac{w_{at}}{u_{at}}\le\frac{w_{x}}{u_{x}}}u_{at}>x.\label{eq:17}
\end{equation}
Then, we can naturally define $\beta_{a}$ as the exponential average
over ratios $\frac{w_{x}}{u_{x}}$. As before, we also assume that
there exists a dummy advertiser that only receives impressions of
zero value-size ratio and that all advertisers are initially filled
up with impressions of zero value.
\begin{algorithm}
\smallskip{}
\noindent\begin{minipage}[t]{1\columnwidth}%
\begin{algorithmic}[1]
\STATE {\bf Input:} Robustness-consistency trade-off parameter $\alpha \in [1,\infty)$
\STATE For each advertiser $a$, initialize $\beta_a \gets 0$ and fill up $a$ with zero-value impressions
\FORALL{arriving impressions $t$}
\STATE $\apred \gets \EST(t)$
\STATE $\aalgo \gets \arg\max_{a}\{w_{at}-u_{at}\beta_{a}\}$ 
\IF{$\alpha\left(w_{\apred,t}-u_{\apred,t}\beta_{\apred}\right)\ge w_{\aalgo,t}-u_{\aalgo,t}\beta_{\aalgo}$} \label{alg:selection-rule-1}
\STATE $a \gets \apred$
\ELSE
\STATE $a \gets \aalgo$
\ENDIF
\STATE Dispose of impressions with least value-size ratio currently allocated to $a$ until there is $u_{at}$ of free space and allocate $t$ to $a$
\STATE Let $\frac{w_x}{u_x}$ as in (\ref{eq:17}) and update $\displaystyle \beta_{a}\gets\frac{\alpha}{e^{\alpha}-1}\int_{U_{a}-1}^{U_{a}}\frac{w_{x}}{u_{x}}e^{\alpha(U_{a}-x)}dx$
\ENDFOR
\end{algorithmic}%
\end{minipage}\smallskip{}

\caption{Exponential Averaging with Predictions for GAP \label{alg:exp-avg-1}}
\end{algorithm}

\subsection{Robustness}
\begin{thm}
Algorithm \ref{alg:exp-avg} has a robustness of
\[
\frac{\ALG}{\OPT}\ge\frac{e^{\alpha}-1}{\alpha e^{\alpha}}
\]
\end{thm}

\begin{proof}
Assume we assign impression $t$ to advertiser $a$ while disposing
of some impressions to make space. We will bound the dual increase
as a multiple of the primal increase. We now assume that after allocating
$t$ to $a$, it becomes the impression with highest value-size ratio
(a general proof follows analogously to the proof of robustness for
Display Ads in Section \ref{subsec:omitted-robustness}). The primal
increase is simply
\[
\Delta P=\int_{U_{a}-u_{at}}^{U_{a}}\frac{w_{x}}{u_{x}}dx-\int_{U_{a}-1-u_{at}}^{U_{a}-1}\frac{w_{x}}{u_{x}}dx=w_{at}-\int_{U_{a}-1-u_{at}}^{U_{a}-1}\frac{w_{x}}{u_{x}}dx.
\]
At the same time,
\begin{align*}
\beta_{a}^{(t)} & =\frac{\alpha}{e^{\alpha}-1}\int_{U_{a}-1}^{U_{a}}\frac{w_{x}}{u_{x}}e^{\alpha(U_{a}-x)}dx\\
 & =\frac{\alpha}{e^{\alpha}-1}\left(\int_{U_{a}-1-u_{at}}^{U_{a}-u_{at}}\frac{w_{x}}{u_{x}}e^{\alpha(U_{a}-x)}dx+\int_{U_{a}-u_{at}}^{U_{a}}\frac{w_{x}}{u_{x}}e^{\alpha(U_{a}-x)}dx-\int_{U_{a}-1-u_{at}}^{U_{a}-1}\frac{w_{x}}{u_{x}}e^{\alpha(U_{a}-x)}dx\right)\\
 & =\frac{\alpha}{e^{\alpha}-1}\left(e^{\alpha u_{at}}\int_{U_{a}-1-u_{at}}^{U_{a}-u_{at}}\frac{w_{x}}{u_{x}}e^{\alpha(U_{a}-x-u_{at})}dx+w_{at}-\int_{U_{a}-1-u_{at}}^{U_{a}-1}\frac{w_{x}}{u_{x}}e^{\alpha(U_{a}-x)}dx\right)\\
 & =e^{\alpha u_{at}}\beta_{a}^{(t-1)}+\frac{\alpha}{e^{\alpha}-1}\left(w_{at}-\int_{U_{a}-1-u_{at}}^{U_{a}-1}\frac{w_{x}}{u_{x}}e^{\alpha(U_{a}-x)}dx\right)
\end{align*}
We set $z_{t}=\alpha\left(w_{at}-u_{at}\beta_{a}^{(t-1)}\right)$
and obtain, since $e^{\alpha u_{at}}-1=\alpha u_{at}$ due to $u_{at}\to0$,
\begin{align*}
\Delta D & =\beta_{a}^{(t)}-\beta_{a}^{(t-1)}+z_{t}\\
 & =\left(e^{\alpha u_{at}}-1\right)\beta_{a}^{(t-1)}+\frac{\alpha}{e^{\alpha}-1}\left(w_{at}-\int_{U_{a}-1-u_{at}}^{U_{a}-1}\frac{w_{x}}{u_{x}}e^{\alpha(U_{a}-x)}dx\right)+\alpha\left(w_{at}-u_{at}\beta_{a}^{(t-1)}\right)\\
 & =\alpha u_{at}\beta_{a}^{(t-1)}+\frac{\alpha}{e^{\alpha}-1}\left(w_{at}-\int_{U_{a}-1-u_{at}}^{U_{a}-1}\frac{w_{x}}{u_{x}}e^{\alpha(U_{a}-x)}dx\right)+\alpha\left(w_{at}-u_{at}\beta_{a}^{(t-1)}\right)\\
 & =\frac{\alpha e^{\alpha}}{e^{\alpha}-1}w_{at}-\frac{\alpha e^{\alpha}}{e^{\alpha}-1}\int_{U_{a}-1-u_{at}}^{U_{a}-1}\frac{w_{x}}{u_{x}}dx\\
 & =\frac{\alpha e^{\alpha}}{e^{\alpha}-1}\Delta P.
\end{align*}

\end{proof}

\subsection{Consistency}
\begin{thm}
\label{thm:main-1} Algorithm \ref{alg:exp-avg} has a consistency
of
\[
\frac{\ALG}{\EST}\ge\left(1+\frac{1}{e^{\alpha}-1}\max\left\{ \frac{1}{\alpha}\left(e^{\alpha}-\frac{e^{\alpha}-1}{\alpha}\right),\alpha\right\} \right)^{-1}.
\]
\end{thm}

As before, we split the impressions $t$ based on whether the algorithm
followed the prediction or not. If the algorithm ignores the prediction,
we can use that $\alpha\left(w_{\apred,t}-u_{\apred,t}\beta_{\apred}\right)\le w_{\aalgo,t}-u_{\aalgo,t}\beta_{\aalgo}$
due to Line \ref{alg:selection-rule-1} in Algorithm \ref{alg:exp-avg-1}.
With a similar calculation, we obtain
\begin{multline*}
\EST=\sum_{a}\EST_{a}\\
=\sum_{a}\Bigg(\sum_{t\in\P_{a}\cap\X_{a}}w_{at}+\frac{1}{\alpha}\sum_{t\in\X_{a}\setminus\P_{a}}w_{at}-\frac{1}{\alpha}\sum_{t\in\X_{a}\setminus\P_{a}}u_{at}\beta_{a}^{(t-1)}+\sum_{t\in\P_{a}\setminus\X_{a}}u_{at}\beta_{a}^{(t-1)}\Bigg).
\end{multline*}
Once again, we fix an advertiser $a$. Let $\rho_{a}\coloneqq\sum_{t\in\P_{a}\cap\X_{a}}u_{at}$
so that we can bound
\[
\sum_{t\in\P_{a}\setminus\X_{a}}u_{at}\beta_{a}^{(t-1)}\le\left(1-\rho_{a}\right)\beta_{a}^{(T)}.
\]
We still have to argue that the worst-case is when all impressions
are ordered such that their value-size ratios are non-decreasing and
the impressions in $\P_{a}\cap\X_{a}$ are the ones with maximum value-size
ratio among $\X_{a}$. The latter is obvious as it can only increase
the value of $\EST$, so it remains to show that the non-decreasing
value-size ordering minimizes the third sum in $\EST_{a}$ (the first
two sums are invariant under reordering). To this end, note that the
value of $\beta_{a}$ for GAP is the limit of $\beta_{a}$ for Display
Ads in the following sense: For positive $\epsilon\to0$, we can split
each GAP-impression $t\in\X_{a}$ into $\frac{u_{t}}{\epsilon}$ identical
Display Ads-impressions with value $\frac{w_{t}}{u_{t}}$, while assuming
a budget of $1/\epsilon$. Then, the GAP $\beta_{a}$ and Display
Ads $\beta_{a}$ are identical. As we know from Display Ads, the worst
case is achieved when the Display Ads-impressions with value $\frac{w_{t}}{u_{t}}$
are in non-decreasing order. In this ordering, consecutive Display-Ads
impressions with identical value $\frac{w_{t}}{u_{t}}$ still correspond
to the same GAP-impression $t$, so we also know that this ordering
is the worst-case for GAP. We may therefore assume that the impressions
are ordered such that their value-size ratios are non-decreasing.
As such, we obtain
\[
\beta_{a}^{(t)}=\frac{\alpha}{e^{\alpha}-1}\int_{U_{a}^{(t)}-1}^{U_{a}^{(t)}}\frac{w_{x}}{u_{x}}e^{\alpha(U_{a}^{(t)}-x)}dx=\int_{y-1}^{y}\frac{w_{x}}{u_{x}}e^{\alpha(y-x)}dx\eqqcolon\beta_{a}^{(y)}
\]
where $y=U_{a}^{(t)}$. Combined with the fact that $\P_{a}\cap\X_{a}$
are last impressions in $\X_{a}$, we can now write
\begin{multline*}
\sum_{t\in\P_{a}\cap\X_{a}}w_{at}+\frac{1}{\alpha}\sum_{t\in\X_{a}\setminus\P_{a}}w_{at}-\frac{1}{\alpha}\sum_{t\in\X_{a}\setminus\P_{a}}u_{at}\beta_{a}^{(t-1)}\\
=\int_{U_{a}-\rho}^{U_{a}}\frac{w_{x}}{u_{x}}dx+\frac{1}{\alpha}\int_{0}^{U_{a}-\rho_{a}}\frac{w_{x}}{u_{x}}dx-\frac{1}{\alpha}\int_{0}^{U_{a}-\rho_{a}}\beta_{a}^{(x)}dx
\end{multline*}
This helps us to compute $\beta_{a}^{(x)}$ in $\EST_{a}$ and rewrite
the whole term as a linear combination of value-size ratios.
\begin{lem}
\label{lem:1-1} We have
\[
\EST_{a}\le\int_{U_{a}-1}^{U_{a}-\rho_{a}}\frac{w_{x}}{u_{x}}\phi_{x}dx+\int_{U_{a}-\rho_{a}}^{U_{a}}\frac{w_{x}}{u_{x}}\psi_{x}dx+\frac{w_{U_{a}-1}}{u_{U_{a}-1}}\Omega_{a}
\]
where
\begin{align*}
\phi_{x} & \coloneqq\left(1-\rho_{a}\right)\frac{\alpha}{e^{\alpha}-1}e^{\alpha(U_{a}-x)}+\frac{1}{\alpha}\frac{e^{\alpha}-e^{\alpha(U_{a}-\rho_{a}-x)}}{e^{\alpha}-1}\\
\psi_{x} & \coloneqq1+\left(1-\rho_{a}\right)\frac{\alpha}{e^{\alpha}-1}e^{\alpha(U_{a}-x)}\\
\Omega_{a} & \coloneqq\frac{1}{\alpha}\frac{1}{e^{\alpha}-1}\left(\rho_{a}e^{\alpha}-\frac{1}{\alpha}\left(e^{\alpha}-e^{\alpha(1-\rho_{a})}\right)\right)
\end{align*}
\end{lem}

\begin{proof}
We rewrite the third sum in $\EST_{a}$ to
\begin{align*}
 & \int_{0}^{U_{a}-\rho_{a}}\beta_{a}^{(x)}dx\\
 & =\frac{\alpha}{e^{\alpha}-1}\int_{0}^{U_{a}-\rho_{a}}\int_{x-1}^{x}\frac{w_{y}}{u_{y}}e^{\alpha(x-y)}dydx\\
 & =\frac{\alpha}{e^{\alpha}-1}\int_{0}^{U_{a}-\rho_{a}}\frac{w_{y}}{u_{y}}\int_{0}^{\min\left\{ 1,U_{a}-\rho_{a}-y\right\} }e^{\alpha x}dxdy\\
 & =\frac{\alpha}{e^{\alpha}-1}\int_{0}^{U_{a}-1-\rho_{a}}\frac{w_{y}}{u_{y}}\int_{0}^{1}e^{\alpha x}dxdy+\frac{\alpha}{e^{\alpha}-1}\int_{U_{a}-1-\rho_{a}}^{U_{a}-\rho_{a}}\frac{w_{y}}{u_{y}}\int_{0}^{U_{a}-\rho_{a}-y}e^{\alpha x}dxdy\\
 & =\int_{0}^{U_{a}-1-\rho_{a}}\frac{w_{y}}{u_{y}}dy+\frac{1}{e^{\alpha}-1}\int_{U_{a}-1-\rho_{a}}^{U_{a}-\rho_{a}}\frac{w_{y}}{u_{y}}\left(e^{\alpha(U_{a}-\rho_{a}-y)}-1\right)dy
\end{align*}
where for the last equality, we simply evaluated the integral. Using
this in place of the second sum in $\EST_{a}$ cancels out most of
the terms of the second sum:
\begin{align}
 & \int_{0}^{U_{a}-\rho_{a}}\frac{w_{x}}{u_{x}}dx-\int_{0}^{U_{a}-\rho_{a}}\beta_{a}^{(x)}dx\nonumber \\
 & =\int_{0}^{U_{a}-\rho_{a}}\frac{w_{x}}{u_{x}}dx-\int_{0}^{U_{a}-1-\rho_{a}}\frac{w_{y}}{u_{y}}dy-\frac{1}{e^{\alpha}-1}\int_{U_{a}-1-\rho_{a}}^{U_{a}-\rho_{a}}\frac{w_{y}}{u_{y}}\left(e^{\alpha(U_{a}-\rho_{a}-y)}-1\right)dy\nonumber \\
 & =\int_{U_{a}-1-\rho_{a}}^{U_{a}-\rho_{a}}\frac{w_{y}}{u_{y}}\left(1-\frac{e^{\alpha(U_{a}-\rho_{a}-y)}-1}{e^{\alpha}-1}\right)dy\nonumber \\
 & =\int_{U_{a}-1-\rho_{a}}^{U_{a}-\rho_{a}}\frac{w_{y}}{u_{y}}\frac{e^{\alpha}-e^{\alpha(U_{a}-\rho_{a}-y)}}{e^{\alpha}-1}dy\nonumber \\
 & =\int_{U_{a}-1}^{U_{a}-\rho_{a}}\frac{w_{y}}{u_{y}}\frac{e^{\alpha}-e^{\alpha(U_{a}-\rho_{a}-y)}}{e^{\alpha}-1}dy+\int_{U_{a}-1-\rho_{a}}^{U_{a}-1}\frac{w_{y}}{u_{y}}\frac{e^{\alpha}-e^{\alpha(U_{a}-\rho_{a}-y)}}{e^{\alpha}-1}dy.\label{eq:11}
\end{align}
We upper bound the third sum
\begin{align}
\frac{1}{\alpha}\int_{U_{a}-1-\rho_{a}}^{U_{a}-1}\frac{w_{y}}{u_{y}}\frac{e^{\alpha}-e^{\alpha(U_{a}-\rho_{a}-y)}}{e^{\alpha}-1}dy & \le\frac{w_{U_{a}-1}}{u_{U_{a}-1}}\frac{1}{\alpha}\int_{U_{a}-1-\rho_{a}}^{U_{a}-1}\frac{e^{\alpha}-e^{\alpha(U_{a}-\rho_{a}-y)}}{e^{\alpha}-1}dy\nonumber \\
 & =\frac{w_{U_{a}-1}}{u_{U_{a}-1}}\frac{1}{\alpha}\frac{1}{e^{\alpha}-1}\left(\rho_{a}e^{\alpha}-\int_{1-\rho_{a}}^{1}e^{\alpha y}dy\right)\nonumber \\
 & =\frac{w_{U_{a}-1}}{u_{U_{a}-1}}\underbrace{\frac{1}{\alpha}\frac{1}{e^{\alpha}-1}\left(\rho_{a}e^{\alpha}-\frac{1}{\alpha}\left(e^{\alpha}-e^{\alpha(1-\rho_{a})}\right)\right)}_{=\Omega_{a}}\label{eq:18}
\end{align}
Furthermore,
\begin{equation}
\left(1-\rho_{a}\right)\beta_{a}^{(U_{a})}=\left(1-\rho_{a}\right)\frac{\alpha}{e^{\alpha}-1}\int_{U_{a}-1}^{U_{a}}\frac{w_{x}}{u_{x}}e^{\alpha(U_{a}-x)}dx\label{eq:19}
\end{equation}
Combining (\ref{eq:11}), (\ref{eq:18}), and (\ref{eq:19}) and grouping
terms yields
\begin{align*}
 & \int_{U_{a}-\rho_{a}}^{U_{a}}\frac{w_{x}}{u_{x}}dx+\frac{1}{\alpha}\int_{U_{a}-1}^{U_{a}-\rho_{a}}\frac{w_{y}}{u_{y}}\frac{e^{\alpha}-e^{\alpha(U_{a}-\rho_{a}-y)}}{e^{\alpha}-1}dy+\frac{w_{U_{a}-1}}{u_{U_{a}-1}}\Omega_{a}\\
 & \quad+\left(1-\rho_{a}\right)\frac{\alpha}{e^{\alpha}-1}\int_{U_{a}-1}^{U_{a}}\frac{w_{x}}{u_{x}}e^{\alpha(U_{a}-x)}dx\\
 & =\int_{U_{a}-1}^{U_{a}-\rho_{a}}\frac{w_{x}}{u_{x}}\underbrace{\left(\left(1-\rho_{a}\right)\frac{\alpha}{e^{\alpha}-1}e^{\alpha(U_{a}-x)}+\frac{1}{\alpha}\frac{e^{\alpha}-e^{\alpha(U_{a}-\rho_{a}-x)}}{e^{\alpha}-1}\right)}_{=\phi_{x}}dx\\
 & \quad+\int_{U_{a}-\rho_{a}}^{U_{a}}\frac{w_{x}}{u_{x}}\underbrace{\left(1+\left(1-\rho_{a}\right)\frac{\alpha}{e^{\alpha}-1}e^{\alpha(U_{a}-x)}\right)}_{=\psi_{x}}dx+\frac{w_{U_{a}-1}}{u_{U_{a}-1}}\Omega_{a}.
\end{align*}
\end{proof}
Analogously to Display Ads, we define
\[
\Phi_{a}\coloneqq\int_{U_{a}-1}^{U_{a}-\rho_{a}}\phi_{x}dx\qquad\textrm{and}\qquad\Psi_{a}\coloneqq\int_{U_{a}-\rho}^{U_{a}}\psi_{x}dx
\]
and the total coefficient $\tau_{a}\coloneqq\Phi_{a}+\Psi_{a}+\Omega_{a}$
which by a calculation similar to Lemma \ref{lem:avgfac} can be shown
to be 
\[
\tau_{a}=1+\frac{1}{e^{\alpha}-1}\frac{1}{\alpha}\left(e^{\alpha}-\frac{e^{\alpha}-1}{\alpha}\right).
\]

\begin{lem}
\label{lem:com-1} We have
\[
\EST\le\max\left\{ \tau_{a},\frac{\Psi_{a}}{\rho_{a}}\right\} \ALG
\]
if $\rho_{a}>0$ and otherwise,
\[
\EST\le\tau_{a}\ALG.
\]
\end{lem}

\begin{proof}
Again, let
\begin{align*}
\bar{w}_{\Phi} & \coloneqq\frac{1}{1-\rho_{a}}\int_{U_{a}-1}^{U_{a}-\rho_{a}}\frac{w_{x}}{u_{x}}dx\\
\bar{w}_{\Psi} & \coloneqq\frac{1}{\rho_{a}}\int_{U_{a}-\rho_{a}}^{U_{a}}\frac{w_{x}}{u_{x}}dx
\end{align*}
be the average coefficients on the intervals $\left[U_{a}-1,U_{a}-\rho_{a}\right]$
and $\left[U_{a}-\rho_{a},U_{a}\right]$, respectively. The latter
coefficients are still decreasing as 

\[
\psi_{x}=1+\underbrace{\left(1-\rho_{a}\right)\frac{\alpha}{e^{\alpha}-1}}_{\ge0}e^{\alpha(U_{a}-x)}
\]
so we can bound the linear combination 
\[
\int_{U_{a}-\rho_{a}}^{U_{a}}\frac{w_{x}}{u_{x}}\psi_{x}dx\le\bar{w}_{\Psi}\int_{U_{a}-\rho}^{U_{a}}\psi_{x}dx=\bar{w}_{\Psi}\Psi_{a}.
\]
However, $\phi_{x}$ is not always decreasing which can be seen by
rearranging
\begin{align*}
\phi_{x} & =\left(1-\rho_{a}\right)\frac{\alpha}{e^{\alpha}-1}e^{\alpha(U_{a}-x)}+\frac{1}{\alpha}\frac{e^{\alpha}-e^{\alpha(U_{a}-\rho_{a}-x)}}{e^{\alpha}-1}\\
 & =\frac{1}{e^{\alpha}-1}\left(\left(1-\rho_{a}\right)\alpha-\frac{1}{\alpha}e^{-\alpha\rho_{a}}\right)e^{\alpha(U_{a}-x)}+\frac{1}{\alpha}\frac{1}{e^{\alpha}-1}e^{\alpha}
\end{align*}
We observe that $\phi_{x}$ is decreasing if $\left(1-\rho_{a}\right)\alpha$
is at least $\frac{1}{\alpha}e^{-\alpha\rho_{a}}$, and we analyze
two cases based on the relationship of both terms:
\begin{itemize}
\item $\left(1-\rho_{a}\right)\alpha\ge\frac{1}{\alpha}e^{-\alpha\rho_{a}}$:
We have $\int_{U_{a}-1}^{U_{a}-\rho_{a}}\frac{w_{x}}{u_{x}}\phi_{x}dx\le\bar{w}_{\Phi}\Phi_{a}$
and thus
\begin{align}
 & \int_{U_{a}-1}^{U_{a}-\rho_{a}}\frac{w_{x}}{u_{x}}\phi_{x}dx+\int_{U_{a}-\rho}^{U_{a}}\frac{w_{x}}{u_{x}}\psi_{x}dx+\frac{w_{U_{a}-1}}{u_{U_{a}-1}}\Omega_{a}\nonumber \\
 & \le\bar{w}_{\Phi}\Phi_{a}+\bar{w}_{\Psi}\Psi_{a}+w_{a,I_{a}-B_{a}}\Omega_{a}\nonumber \\
 & \le\bar{w}_{\Phi}\left(\Phi_{a}+\Omega_{a}\right)+\bar{w}_{\Psi}\Psi_{a}\nonumber \\
 & =\bar{w}_{\Phi}\left(1-\rho_{a}\right)\tau_{a}+\bar{w}_{\Phi}\left(\Phi_{a}+\Omega_{a}-\left(1-\rho_{a}\right)\tau_{a}\right)+\bar{w}_{\Psi}\Psi_{a}\nonumber \\
 & =\int_{U_{a}-1}^{U_{a}-\rho_{a}}\frac{w_{x}}{u_{s}}\tau_{a}dx+\bar{w}_{\Phi}\left(\Phi_{a}+\Omega_{a}-\left(1-\rho_{a}\right)\tau_{a}\right)+\bar{w}_{\Psi}\Psi_{a}\label{eq:12}
\end{align}
\item $\left(1-\rho_{a}\right)\alpha\le\frac{1}{\alpha}e^{-\alpha\rho_{a}}$:
We can still show that $\phi_{x}\le\tau_{a}$ as
\begin{multline*}
\phi_{x}=\left(1-\rho_{a}\right)\frac{\alpha}{e^{\alpha}-1}e^{\alpha(U_{a}-x)}+\frac{1}{\alpha}\frac{e^{\alpha}-e^{\alpha(U_{a}-\rho_{a}-x)}}{e^{\alpha}-1}\\
\le1+\frac{1}{e^{\alpha}-1}\frac{1}{\alpha}\left(e^{\alpha}-\frac{e^{\alpha}-1}{\alpha}\right)=\tau_{a}
\end{multline*}
\[
\iff\underbrace{\left(\left(1-\rho_{a}\right)\alpha-\frac{1}{\alpha}e^{-\alpha\rho_{a}}\right)}_{\le0}\underbrace{e^{\alpha(U_{a}-x)}}_{\ge0}\le e^{\alpha}-1-\frac{1}{\alpha}\frac{e^{\alpha}-1}{\alpha}=\underbrace{\left(1-\frac{1}{\alpha^{2}}\right)}_{\ge0}\underbrace{\left(e^{\alpha}-1\right)}_{\ge0}.
\]
Therefore,
\begin{align}
 & \int_{U_{a}-1}^{U_{a}-\rho_{a}}\frac{w_{x}}{u_{x}}\phi_{x}dx+\int_{U_{a}-\rho}^{U_{a}}\frac{w_{x}}{u_{x}}\psi_{x}dx+\frac{w_{U_{a}-1}}{u_{U_{a}-1}}\Omega_{a}\nonumber \\
 & \le\int_{U_{a}-1}^{U_{a}-\rho_{a}}\frac{w_{x}}{u_{x}}\phi_{x}dx+\bar{w}_{\Psi}\Psi_{a}+\frac{w_{U_{a}-1}}{u_{U_{a}-1}}\Omega_{a}\nonumber \\
 & =\int_{U_{a}-1}^{U_{a}-\rho_{a}}\frac{w_{x}}{u_{x}}\tau_{a}dx-\int_{U_{a}-1}^{U_{a}-\rho_{a}}\frac{w_{x}}{u_{x}}\left(\tau_{a}-\phi_{x}\right)dx+\bar{w}_{\Psi}\Psi_{a}+\frac{w_{U_{a}-1}}{u_{U_{a}-1}}\Omega_{a}\nonumber \\
 & \le\int_{U_{a}-1}^{U_{a}-\rho_{a}}\frac{w_{x}}{u_{x}}\tau_{a}dx-\int_{U_{a}-1}^{U_{a}-\rho_{a}}\frac{w_{U_{a}-1}}{u_{U_{a}-1}}\left(\tau_{a}-\phi_{x}\right)dx+\bar{w}_{\Psi}\Psi_{a}+\frac{w_{U_{a}-1}}{u_{U_{a}-1}}\Omega_{a}\nonumber \\
 & =\int_{U_{a}-1}^{U_{a}-\rho_{a}}\frac{w_{x}}{u_{x}}\tau_{a}dx+\frac{w_{U_{a}-1}}{u_{U_{a}-1}}\left(\Phi_{a}+\Omega_{a}-\left(1-\rho_{a}\right)\tau_{a}\right)+\bar{w}_{\Psi}\Psi_{a}\label{eq:13}
\end{align}
\end{itemize}
In both cases (\ref{eq:12}) and (\ref{eq:13}), we have shown that
\begin{multline*}
\int_{U_{a}-1}^{U_{a}-\rho_{a}}\frac{w_{x}}{u_{x}}\phi_{x}dx+\int_{U_{a}-\rho}^{U_{a}}\frac{w_{x}}{u_{x}}\psi_{x}dx+\frac{w_{U_{a}-1}}{u_{U_{a}-1}}\Omega_{a}\\
\le\int_{U_{a}-1}^{U_{a}-\rho_{a}}\frac{w_{x}}{u_{x}}\tau_{a}dx+v\left(\Phi_{a}+\Omega_{a}-\left(1-\rho_{a}\right)\tau_{a}\right)+\bar{w}_{\Psi}\Psi_{a}
\end{multline*}
for a $v\le\bar{w}_{\Phi}$. 

\begin{align*}
 & \int_{U_{a}-1}^{U_{a}-\rho_{a}}\frac{w_{x}}{u_{x}}\tau_{a}dx+v\left(\Phi_{a}+\Omega_{a}-\left(1-\rho_{a}\right)\tau_{a}\right)+\bar{w}_{\Psi}\Psi_{a}\\
 & \le\int_{U_{a}-1}^{U_{a}-\rho_{a}}\frac{w_{x}}{u_{x}}\tau_{a}dx+\bar{w}_{\Psi}\max\left\{ \Phi_{a}+\Omega_{a}-\left(1-\rho_{a}\right)\tau_{a},0\right\} +\bar{w}_{\Psi}\Psi_{a}\\
 & =\int_{U_{a}-1}^{U_{a}-\rho_{a}}\frac{w_{x}}{u_{x}}\tau_{a}dx+\bar{w}_{\Psi}\max\left\{ \Phi_{a}+\Psi_{a}+\Omega_{a}-\left(1-\rho_{a}\right)\tau_{a},\Psi_{a}\right\} \\
 & =\int_{U_{a}-1}^{U_{a}-\rho_{a}}\frac{w_{x}}{u_{x}}\tau_{a}dx+\bar{w}_{\Psi}\max\left\{ \rho_{a}\tau_{a},\Psi_{a}\right\} \\
 & \le\tau_{a}\int_{U_{a}-1}^{U_{a}-\rho_{a}}\frac{w_{x}}{u_{x}}dx+\max\left\{ \tau_{a},\frac{\Psi_{a}}{\rho_{a}}\right\} \int_{U_{a}-\rho_{a}}^{U_{a}}\frac{w_{x}}{u_{x}}dx\\
 & \le\max\left\{ \tau_{a},\frac{\Psi_{a}}{\rho_{a}}\right\} \int_{U_{a}-1}^{U_{a}}\frac{w_{x}}{u_{x}}dx
\end{align*}
or $\le\tau_{a}\int_{U_{a}-1}^{U_{a}}\frac{w_{x}}{u_{x}}dx$ if $\rho_{a}=0$
\end{proof}
 Note that for the bound of Lemma \ref{lem:bound}, we did not require
that $\ell_{a}$ is integral. We can thus apply Lemma \ref{lem:bound}
to bound $\max\left\{ \tau_{a},\frac{\Psi_{a}}{\rho_{a}}\right\} $
and obtain the same result, which proves Theorem \ref{thm:main-1}.

\section{Further Experimental Results}

\begin{figure}[t]
\begin{centering}
\par\end{centering}
\begin{centering}
\includegraphics[width=0.85\linewidth]{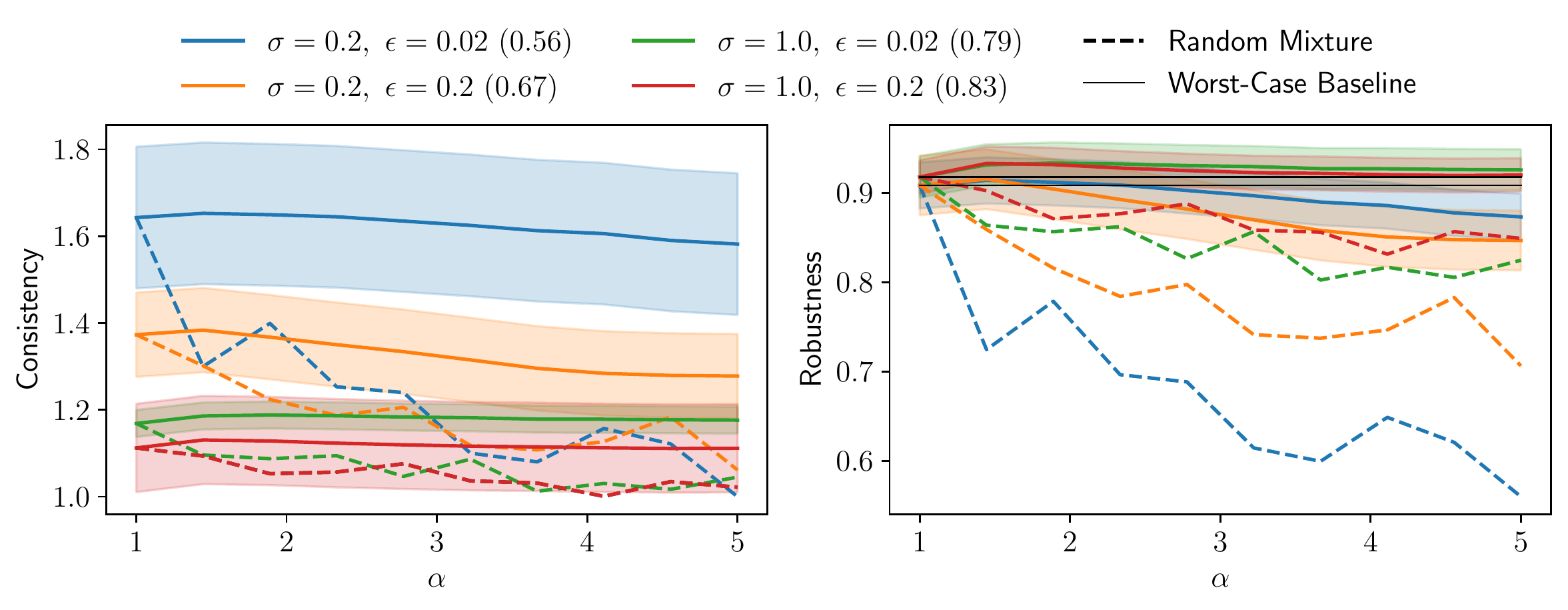}
\par\end{centering}
\caption{\label{fig:synthetic-1-1} Experimental results for varying values
of $\alpha$ on synthetic data with $12$ advertisers and $2000$
impressions of $10$ types, where we report the same quantities as
in Figure \ref{fig:real-world-1}. We use Dual Base predictions for
different $\sigma$ and $\epsilon$. Note that there are two black
lines indicating the performance of the worst-case algorithm without
predictions, corresponding to the datasets with differing $\sigma$. }
\end{figure}

\begin{figure}[t]
\begin{centering}
\includegraphics[width=0.8\linewidth]{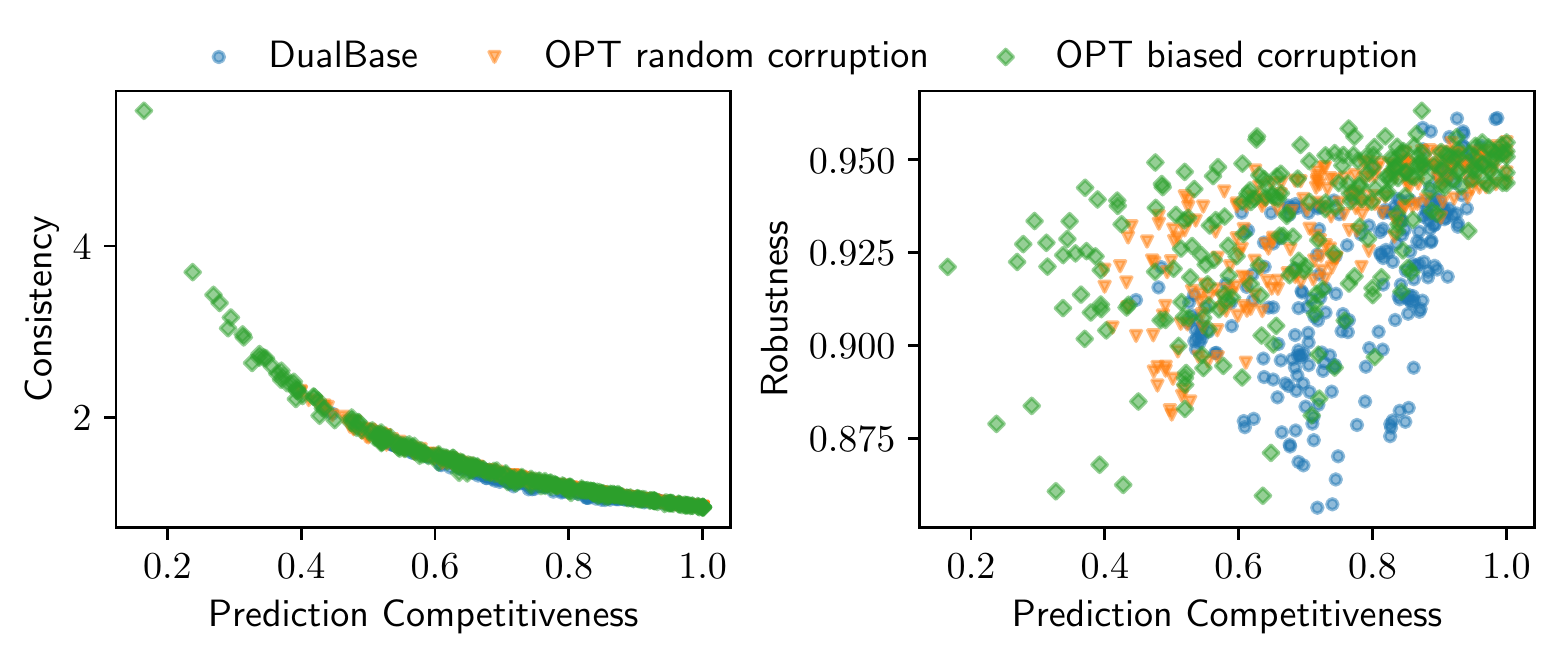}
\par\end{centering}
\begin{centering}
\includegraphics[width=0.8\linewidth]{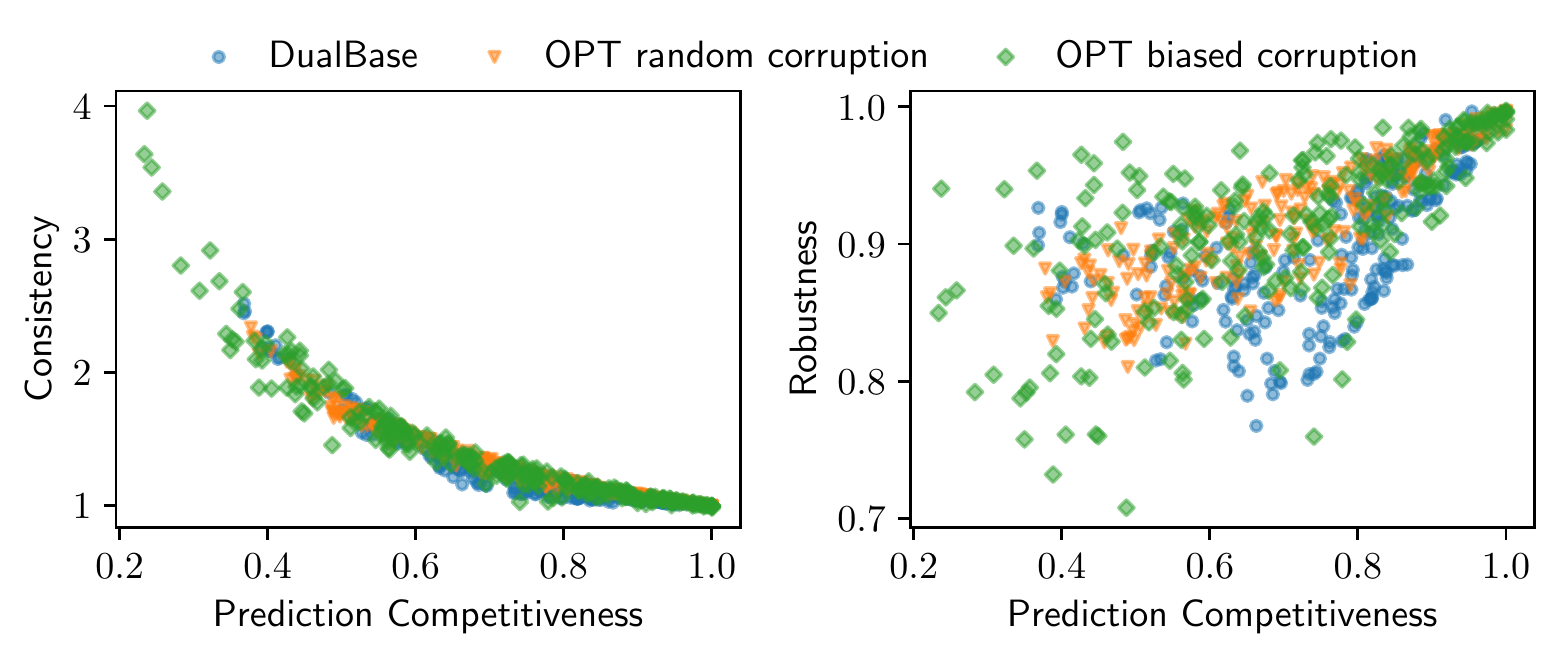}
\par\end{centering}
\caption{\label{fig:synthetic-2-1} Performance for varying prediction quality
with the data from Figure \ref{fig:synthetic-1-1} (top) for $\alpha=2$
(top) and $\alpha=5$ (bottom).}
\end{figure}

\label{sec:apx-exp}

\subsection{Real-World Data}

\label{subsec:yahoo-description}

\paragraph{Description of the Yahoo Dataset}

The original dataset contains impression allocations to 16268 advertisers
throughout 123 days, each tagged with the advertiser that bought the
impression and a set of keyphrases that categorize the user for whom
the impression is displayed. \citet{lavastida21} then consider the
20 most common keyphrases and create an impression type for each non-empty
subset thereof. Whenever an advertiser buys an impression with a certain
set of keyphrases, we assume that all impression types that correspond
to a superset of these keyphrases are relevant for this advertiser,
and that it derives some constant value (say, 1) from this allocation.
At the same time, the number of impressions we create from each impression
type (i.e. the supply) is the number of impression allocations in
the original dataset that show that the impression type is relevant
for an advertiser. As such, we obtain around 2 million impressions.
\citet{lavastida21} try multiple impression orders and budgets for
the advertisers, but due to space constraints we restrict ourselves
to display all impressions of a type at once, in supply-ascending
order. We determine advertisers' budgets by allocating each impression
to one of the advertisers with non-zero valuation uniformly at random
and taking the number of allocated impressions at the end to be the
advertiser's budget.

\subsection{Synthetic Data}

\label{subsec:synthetic}

\paragraph{Results}

Figure \ref{fig:synthetic-1-1-1} shows consistency and robustness
of our algorithm on synthetic data on $\numimps=2000$ impressions
of 10 types and $k=12$ advertisers, for a variation of predictions.
The plot shows the performance for predictions from the optimum solutions
(with varying corruption) and the dual base prediction. Our algorithm
converges to almost perfect consistency and robustness for $\alpha=10$,
given the optimum solution. At the same time, we observe that the
algorithm is robust against both random and biased corruption, as
the robustness does not drop to the prediction's low competitiveness
of around $0.7$. Furthermore, the algorithm performs well in combination
with the dual base prediction for $\epsilon=0.1$ even though the
first $200$ impressions are clearly not representative of all synthetically
generated impressions.

To investigate the our algorithm in conjunction with an easily available
prediction, we also analyze the behavior of the dual base algorithm
for different values of $\sigma$ and $\epsilon$ in Figure \ref{fig:synthetic-1-1}.
The performance of our algorithm under dual base predictions clearly
improves for increasing values of $\sigma$ as impressions become
more evenly distributed across the day. Generally, sampling more impressions
helps but dual base predictions may also lead to a drop in robustness,
and more samples can even lead to a more adversarial prediction, as
we explore further below. Yet, the robustness does still stays above
the prediction's competitiveness in these cases.

Figure \ref{fig:synthetic-2-1} shows consistency and robustness for
different predictions with varying competitiveness on $\alpha\in\left\{ 2,5\right\} $.
We achieve this by varying the fraction $\epsilon\in\left[0,1\right]$
of samples for the dual base algorithm and the corruption rate $p\in\left[0,1\right]$
for random and biased corruptions. For $\alpha=2$, the consistency
exceeds 1 if the prediction is not very good (competitiveness below
$0.9$). The algorithm is not heavily influenced by a bad prediction
since $\alpha=2$ is low, so the total obtained value remains relatively
constant. For $\alpha=5$, the algorithm might however follow the
bad choices of the prediction, so the competitiveness varies more.
As expected, the average robustness decreases for increasing $\alpha$,
but the dual base prediction starts out with a much lower robustness
than the corrupted predictions. The reason for that is that both the
dual base algorithm and exponential averaging make their decisions
based on the discounted gain. Our algorithm might therefore easily
disregard a corrupted prediction as its discounted gain is low (or
even negative), but the dual base prediction looks like a sensible
choice. The dual base algorithm therefore manages to fool the algorithm
for low $\alpha$, while a biased corruption leads to the worst corruption
for larger values of $\alpha$. 

\paragraph{Hard Instances}

\begin{figure}[t]
\begin{centering}
\includegraphics[width=0.9\linewidth]{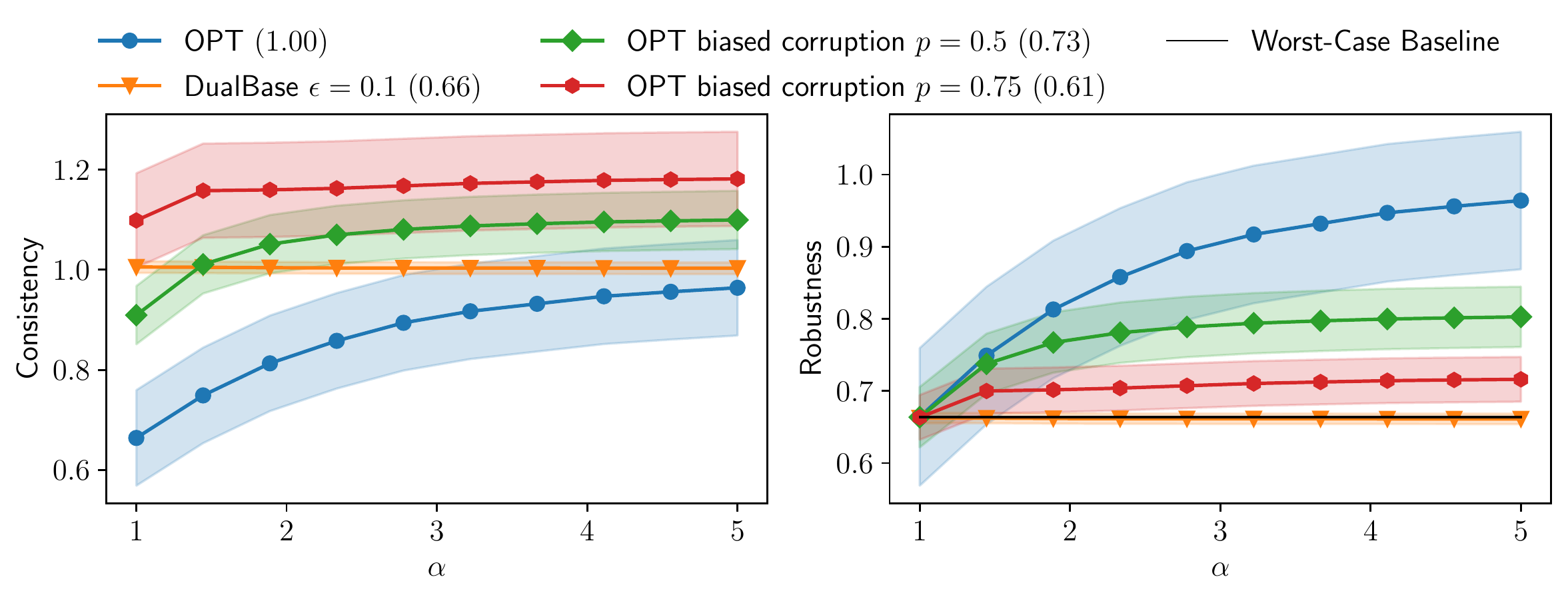}
\par\end{centering}
\caption{\label{fig:worst-case-1} Performance on a worst-case instance with
different predictors.}
\end{figure}

\begin{figure}[tb]
\begin{centering}
\includegraphics[width=0.9\linewidth]{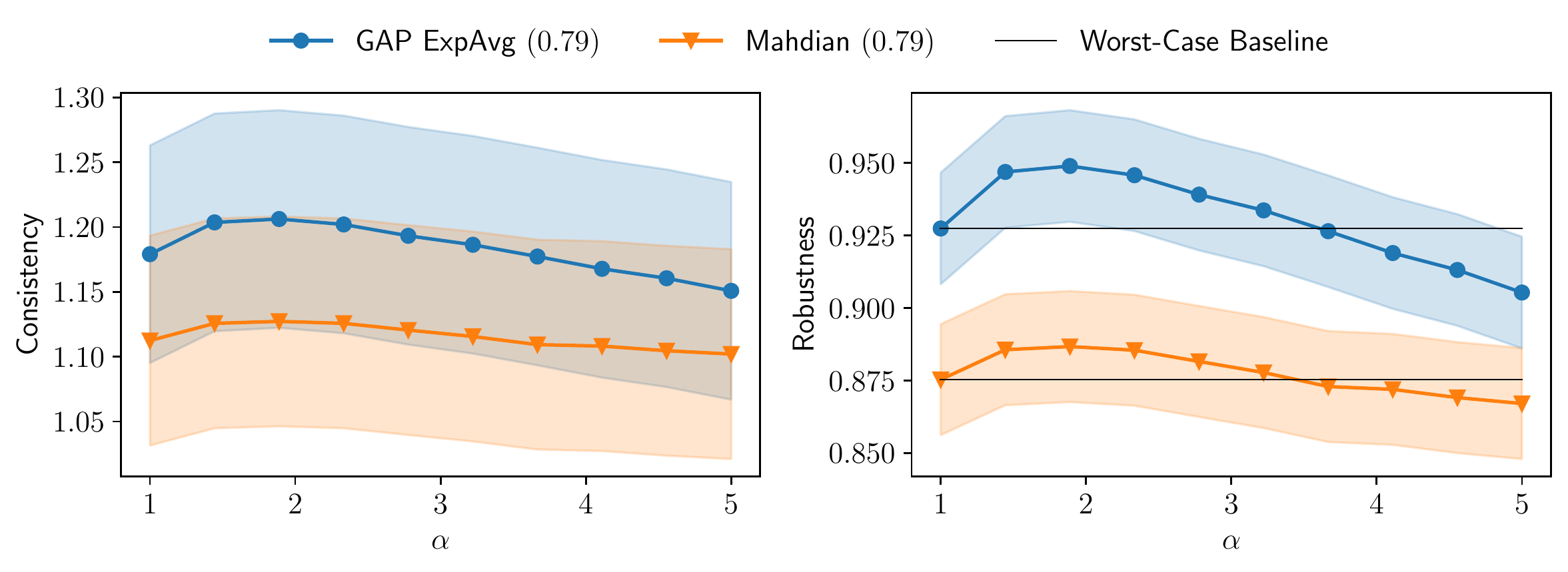}
\par\end{centering}
\caption{\label{fig:adwords} Performance on synthetic Ad Words instances,
compared to the algorithm of \citet{mehta07}. The black lines show
the robustness of two worst-case algorithms without predictions: The
algorithm due to \citet{feldman09} which is the basis for our algorithm,
and the algorithm of \citet{mehta07}, which serves as a basis for
the algorithm of \citet{mahdian07}.}
\end{figure}
We consider the worst-case instance for the Display Ads problem described
in \citet{mehta07}. For $k$ advertisers, we create impressions of
types $r\in\left\{ 1,\dots,k\right\} $. An impression $t$ of type
$r$ has zero value for the first $r-1$ advertisers $w_{1,t}=\cdots=w_{r-1,t}=0$
and value 1 for the following advertisers $w_{r,t}=\cdots=w_{k,t}=1$.
We first show all impressions of type 1, then all impressions of type
2, and so forth. The instance is difficult as the algorithm---not
knowing about future impressions---has to allocate impressions of
a type equally among advertisers that can derive value from this impression
type. As shown by \citet{mehta07}, the competitiveness of the exponential
averaging algorithm reaches $1-\frac{1}{e}$ for $k\to\infty$ on
this instance. 

We evaluate the performance of our algorithm on this worst-case instance
in Figure \ref{fig:worst-case-1}. Providing the optimum solution
as prediction allows the algorithm to quickly ascend to a perfect
robustness of 1. We also consider two (biased) corrupted versions
of this prediction with $p\in\left\{ 50\%,75\%\right\} $. In both
cases, the algorithm still achieves a robustness above the competitiveness
of the prediction. The dual base algorithm cannot deliver meaningful
predictions as it only sees impressions of the first type, which are
clearly not representative of the following impressions by construction.

\subsection{Evaluation of GAP on an Ad Words Instance}

\label{subsec:adwords-gap} With an algorithm for GAP, we can also
solve AdWords instances. This allows us to compare our generalized
algorithm to the algorithm of \citet{mahdian07} under the same predictions.
In Figure \ref{fig:adwords}, we run both algorithms on synthetic
instances from Section \ref{subsec:synthetic} with an optimum prediction
and random corruption ($p=0.5$). Both algorithms seem to have similar
consistency, but our algorithm achieves a better robustness, due to
a different choice of constants in the underlying algorithms.
\end{document}